\documentclass[]{interact}

\usepackage{epstopdf}
\usepackage[caption=false]{subfig}

\usepackage[numbers,sort&compress]{natbib}
\bibpunct[, ]{[}{]}{,}{n}{,}{,}
\renewcommand\bibfont{\fontsize{10}{12}\selectfont}
\makeatletter
\def\NAT@def@citea{\def\@citea{\NAT@separator}}
\makeatother

\theoremstyle{plain}
\newtheorem{theorem}{Theorem}[section]

\theoremstyle{definition}

\theoremstyle{remark}

\usepackage{color}
\usepackage{multirow}

\begin{document}

\articletype{FULL PAPER}

\title{Scratch Team of Single-Rotor Robots and Decentralized Cooperative Transportation with Robot Failure}

\author{
\name{Koshi Oishi\textsuperscript{a}\thanks{CONTACT Koshi Oishi. Email: e1616@mosk.tytlabs.co.jp}, Yasushi Amano\textsuperscript{a}, and Jimbo Tomohiko\textsuperscript{b,a}}
\affil{\textsuperscript{a}Toyota Central R$\&$D Labs., Inc., 41-1, Yokomichi, Nagakute, Aichi 480-1192, Japan; \textsuperscript{b}Toyota Motor Corporation, 1, Toyota-cho, Toyota, Japan}
}

\maketitle

\begin{abstract}
Achieving cooperative transportation by aerial robot teams ensures flexibility regarding payloads and robustness against failures, which has garnered significant attention in recent years.
This study proposes a flexible decentralized controller for robots and the shapes of payloads in a cooperative transport task using multiple single-rotor robots.
The proposed controller is robust to mass and center of mass (COM) fluctuations and robot failures. Moreover, it possesses asymptotic stability against dynamics errors.
Additionally, the controller supports heterogeneous single-rotor robots.
Thus, robots with different specifications and deterioration may be effectively utilized for cooperative transportation.
This performance is particularly effective for robot reuse.
To achieve the aforementioned performance, the controller consists of a parallel structure comprising two controllers: 
a feedback controller, which renders the system strictly positive real, and
a nonlinear controller, which renders the object asymptotic to the target.
First, we confirm cooperative transportation using 8 and 10 robots for two shapes through numerical simulation.
Subsequently, the cooperative transportation of a rectangle payload (with a weight of approximately 3 kg and maximum length of 1.6 m) is demonstrated using a robot team consisting of three types of robots, even under robot failure and fluctuation in the COM.
\end{abstract}


\section{Introduction}\label{sec_1}
With recent advancements in aerial robots, the number of studies on aerial transportation and manipulation has increased significantly \cite{blood,KumarUAV, AEROARMS, OlleroRAL2018, Bonyan2018, AR_TAMS}.
Likewise, the demand for aerial transportation tasks has increased substantially, especially because of the impact of COVID-19 \cite{kumar2020review,wankmuller2021drones, sarker2021robotics}.
Owing to its ability to provide redundancy and system scalability for the transportation of various types of payloads, cooperative transportation using multiple aerial robots has attracted significant attention \cite{Kumar2011, jiang2012inverse, KumarRigid2013, Kumar2018survey, villa2020survey}.
Notably, the center of mass (COM) and mass of the payload should comply with the specifications of the aircraft when using a single aerial robot for transportation.
Cooperative transportation using multiple aerial robots is scalable with respect to the mass of the payload, provided the required number of robots can be added at the necessary positions without interference.
Notably, redundant cooperative transportation is a robust system because, in the event of a single robot failure, other robots can compensate for the thrust.
Furthermore, the practicality of the system improves because robots can easily be plugged in/plugged out when each robot can be controlled through a decentralized method \cite{shibata2021deep, shibata2023deep}.

\begin{figure}[t]
	\centering
	\subfloat[Heterogenous configuration]{
		\includegraphics[width=7cm]{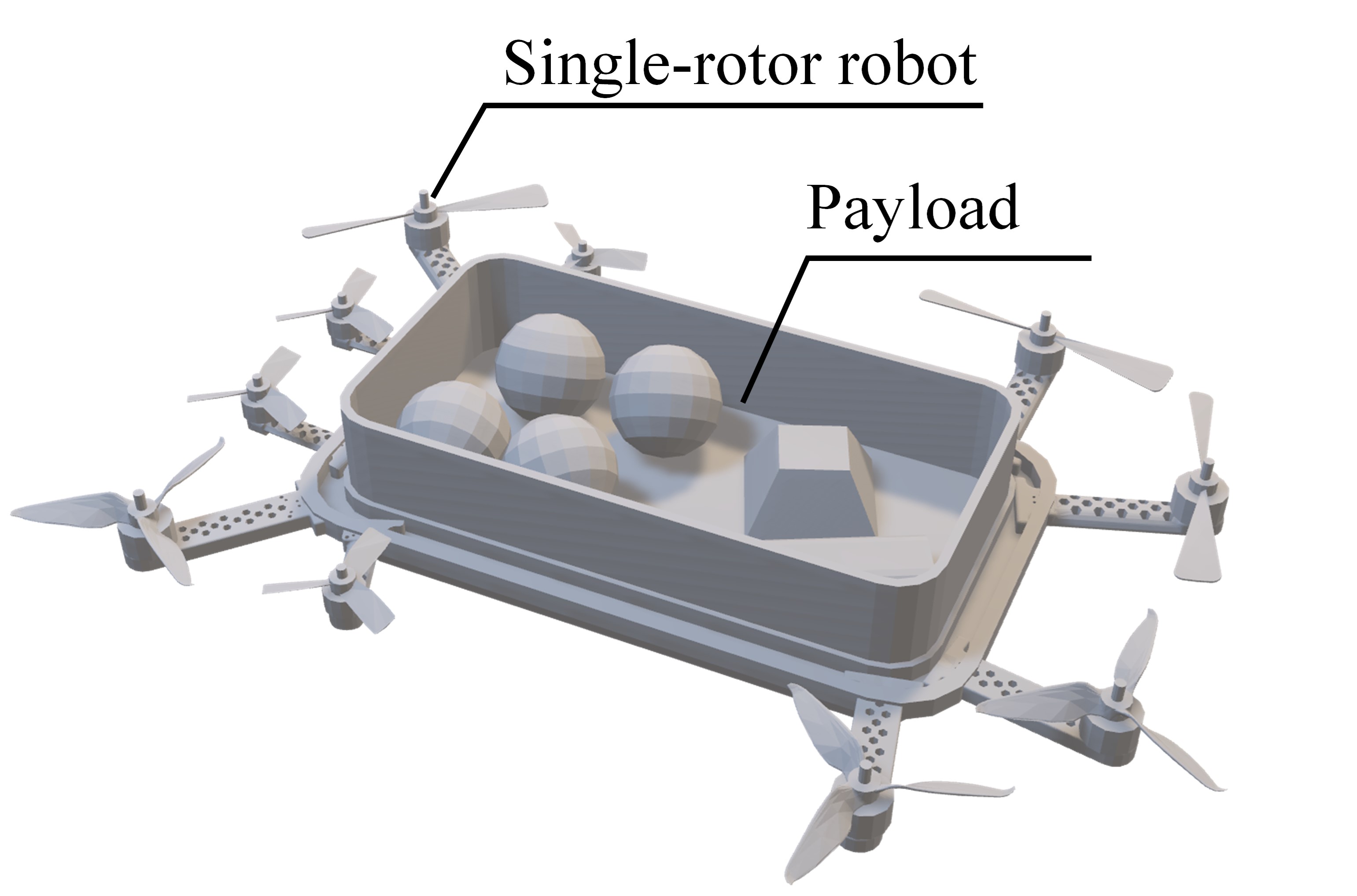}
		\label{fig:fig1a}
	} 
	\subfloat[Real flight with decentralized controller with various disturbances]{
		\includegraphics[width=7cm]{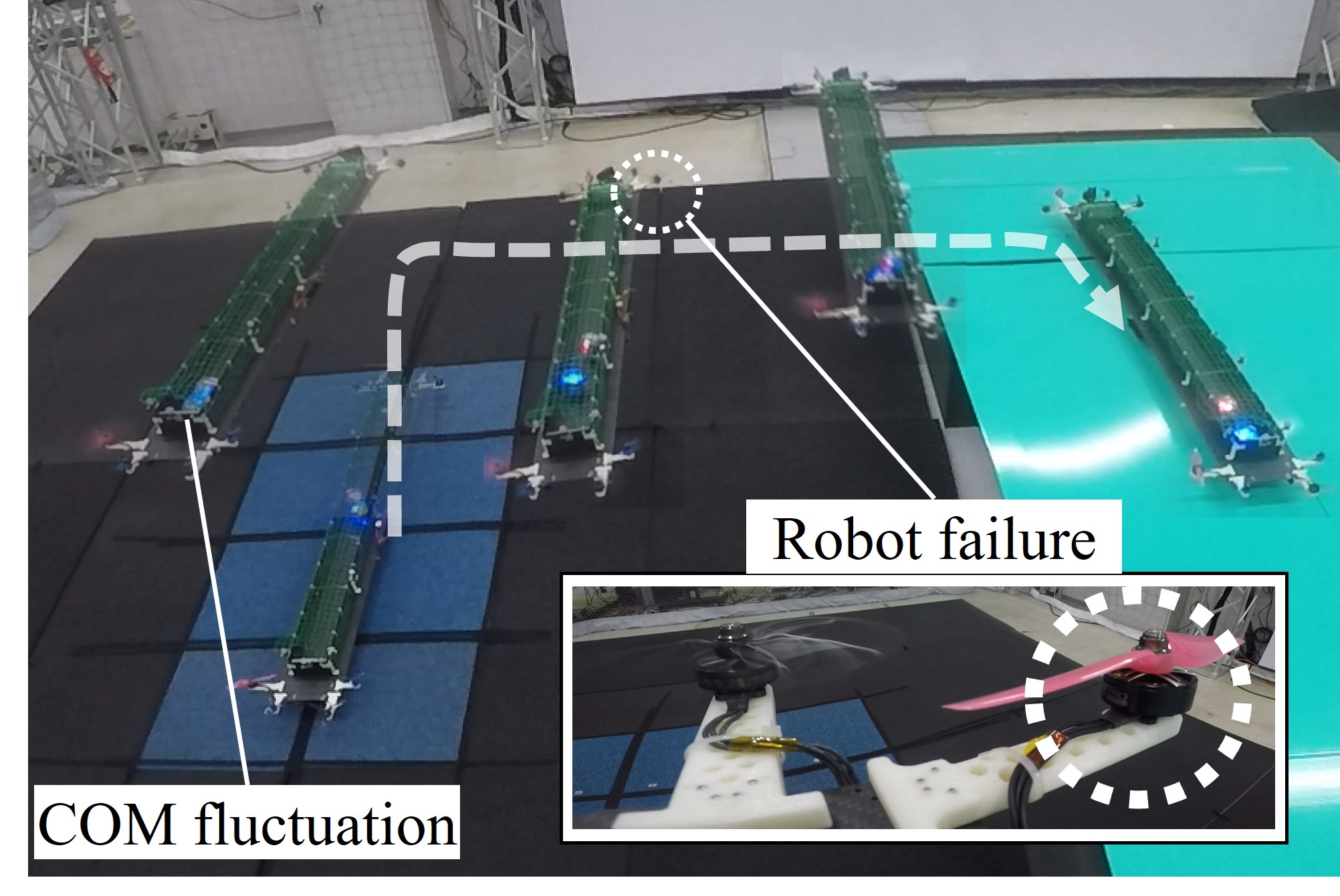}
		\label{fig:fig1b}
	}
	\caption{Concept}
	\label{fig:fig1}
\end{figure}
To adapt multi-robot systems to the real world, the configuration of the robot team should be expanded to include heterogeneous robots \cite{dorigo2013, rizk2019cooperative,WRS2020, cheraghi2022past}.
Studies on heterogeneous robots involve decomposing the overall task and assigning robots designed for different purposes \cite{Simone2015, Philipp2021, Fatemeh2021}, 
as well as achieving robots that perform the same task \cite{pereira2018, lissandrini2020decentralized, zhang2020decentralized}.
Tasks such as carrying a big desk with people of various ages correspond to a same-task situation.
The robot team should manage such tasks because robots exhibiting various performances are expected to be operational in the future.
In this study, we addressed performance such as transportation tasks with heterogeneous aerial robots.
A possible scenario is a transportation task using a robot team of different manufacturers or a team of reuse robots.
If such integration is achieved, aerial cooperative transportation would be practical in the field of robotics.

In our previous study, we constructed a novel multi-copter with a loading platform in the center and actively investigating control for cooperative transportation using multiple single-rotor robots.
Moreover, in \cite{Oishi2021ICRA}, we proposed a decentralized controller that extended an autonomous smooth switching controller (ASSC) \cite{Amano} for homogeneous single-rotor robots.
As a result, multiple single-rotor robots realized aerial cooperative transportation even when COM fluctuation or robot failure occurred.
In this study, we proposed a variable-gain ASSC (VG-ASSC) for a scratch team of heterogeneous single-rotor robots, as shown in Fig. \ref{fig:fig1a}.
We demonstrated that the proposed controller remains asymptotically stable even under dynamic COM fluctuation in a real-world setting via a unique experiment using ball robots, as shown in Fig.~\ref{fig:fig1b}.
The contributions of this study are as follows:

\begin{itemize}
\item We proposed a variable-gain decentralized controller with proven asymptotic stability for an aerial cooperative transportation using a heterogeneous team of robots
\item The performance of the proposed controller was verified using two simulations with different numbers of robots and different shapes of payload
\item The proposed controller was validated via an experiment using a real prototype with three different types of single-rotor robots under dynamic COM fluctuation and robot failure
\end{itemize}

The remainder of this paper proceeds as follows. 
Section \ref{sec_2} describes related work of our proposal;
Section \ref{sec_3} describes the model of the target system;
Section \ref{sec_4} describes the proposed controller;
Section \ref{sec_5} describes the numerical experiments;
Section \ref{sec_6} describes the results of prototype experiments;
and Section \ref{sec_7} presents the advantage of our method and discusses the results of experiments.
Finally, conclusions are presented in Section \ref{sec_8}.

\section{Related work}\label{sec_2}
Cooperative transportation using aerial robots includes robots using cables \cite{Kumar2011,jiang2012inverse, wang2013passivity, shirani2019cooperative} and those using rigid connections \cite{KumarRigid2013, Oishi2021ICRA, corah2017active, wang2018, Oishi2021IROS}.
Rigid configurations can better leverage fault-robustness than cabled configurations, because robot failures in the latter result in a loss in thrust and have a considerable impact on the fall of aerial robots.
Oung et al. \cite{oung2011}, Mu et al. \cite{mu2019universal}, and Chaikalis et al. \cite{chaikalis2023modular} modularized robots to improve scalability.
Recent developments in modular systems include studies on modules that can be detached and function independently in mid-air \cite{bai2022splitflyer}.
Based on these results, we proposed a controller for a team of single-rotor robots rigidly connected to the payload.

Previous studies have mainly focused on controls of aerial cooperative transportation.
These controls can be categorized into centralized and decentralized approaches. 
Wang \cite{wang2013passivity} proposed a centralized robust controller using linear quadratic Gaussian.
This controller was robust against mass fluctuations of approximately $20$ \% in aerial cooperative transportation using cables.
Pereira et al. \cite{pereira2018, pereira2020pose} realized the aerial transportation of a round bar object using two different types of robots.
This controller which was based on proportional-integral-derivative (PID) using an ideal control input guaranteed stability.
Furthermore, Mu et al. \cite{mu2019universal} proposed an adaptive controller that used estimated states using an inertial measurement unit mounted on each robot.
This controller ensures stable flight in various configurations.
Similarly, Chaikalis et al. \cite{chaikalis2023modular} proposed an adaptive controller that was robust to structural flexibilities.
Centralized control in aerial transportation requires equipment and infrastructure to share information because it utilizes all available data. 
Realizing stable flight via decentralized control may reduce the need for extensive information sharing and alleviate the constraints of this equipment.

A decentralized controller with a basic aerial cooperative transportation model of rigid configuration was proposed by Mellinger et al. \cite{KumarRigid2013}.
To achieve decentralized control, they assumed that the positions of each robot are known.
Wang et al. \cite{wang2018} and Shirani et al. \cite{shirani2019cooperative} achieved decentralized control using a symmetric arrangement of robot positions.
Wang et al. used symmetry to alleviate the constraints on the position of a robot and used a wrench-based compensator to guarantee stability.
Shirani et al. used feedback linearization and linear matrix inequality (LMI) to ensure stability.
Cardona \cite{cardona2021robust} utilized a leader--follower controller to achieve decentralized control of heterogeneous robots for aerial cooperative transportation and demonstrated the  
robustness of the controller against fluctuations in the reference.
Furthermore, in our previous study \cite{Oishi2021ICRA}, we proposed a decentralized controller that is robust against disturbances for aerial cooperative transportation using an identical single-rotor robot and integral controller \cite{Amano}.
In the field of robotics, decentralized control has the advantage of robustness against robot failures \cite{dorigo2013, ren2008distributed, kube2000cooperative}.
However, this advantage has not been sufficiently explored in aerial cooperative transportation.
Furthermore, to the best of our knowledge, no studies have specifically focused on heterogeneous robots and their robustness in aerial cooperative transportation involving rigid connections.

In this study, we proposed decentralized cooperative control of heterogeneous multiple single-rotor robots in a real flight environment that includes mass and COM fluctuations as well as robot failure.

\begin{figure}[t]
	\centering
		\includegraphics[keepaspectratio, scale=0.37]{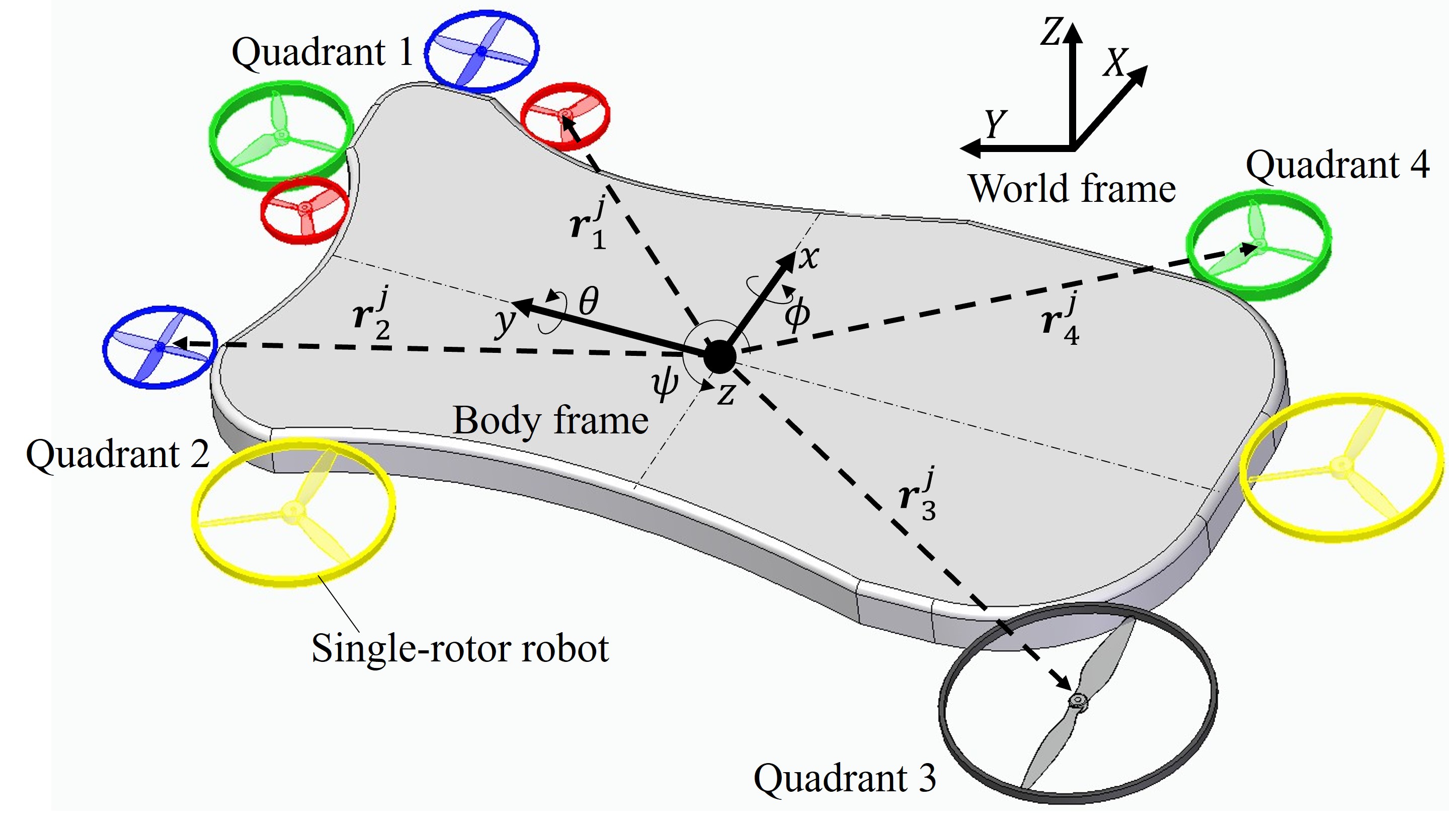}
	  \caption{Model of cooperative transportation using heterogeneous single-rotor robots ($n_1 = 4$, $n_2 = 2$, $n_3 = 1$, $n_2 = 2$)}
		\label{fig:model}
\end{figure}

\section{Flight model}\label{sec_3}
We modeled the dynamics of cooperative transportation using heterogeneous single-rotor robots while focusing on the behavior of the force acting in each quadrant, as shown in Fig. \ref{fig:model}.
Hereafter, $i$ and $j$ indicate a quadrant and robot in each quadrant, respectively.
Let $n_i$ be the number of robots in each quadrant; the dynamics of the hovering state can be approximated as follows \cite{KumarRigid2013}:
\begin{eqnarray}
	\begin{split}
	&m\left[
					\begin{array}{c}
									\ddot x \\
									\ddot y \\
									\ddot z \\                        
					\end{array}
					\right] =
					\left [
					\begin{array}{c}
									mg\theta \\
									-mg\phi \\
									0 \\
					\end{array}
					\right ] +
					\left[
					\begin{array}{c}
									0 \\
									0 \\
									\sum_{i=1}^4 \sum_{j=1}^{n_i} u_i^j \\                       
					\end{array}
					\right ]  -
					\left[
					\begin{array}{c}
									0 \\
									0 \\
									mg\\                       
					\end{array}
					\right ] \\
	&\bm{J}\left[
					\begin{array}{c}
									\ddot \phi \\
									\ddot \theta \\
									\ddot \psi \\                        
					\end{array}
					\right ] =
					\left [
					\begin{array}{c}
									\sum_{i=1}^{4} \sum_{j=1}^{n_i} \left(\bm{r}_{i}^{j} \times \bm{e}_{3}\right)_{12} u_i^j \\[7pt]
									\sum_{i=1}^4 \sum_{j=1}^{n_i} d_i^j c_{\mathrm{q}i}^j u_i^j \\
					\end{array}
					\right ],
	\end{split}
	\label{eq:dy}
\end{eqnarray}
where $x$, $y$, and $z$ are the three-dimensional positions of the payload on a body frame system;
$\phi$, $\theta$, and $\psi$ denote the roll, pitch, and yaw angles of the payload, respectively;
$u_i^j$ denotes the thrust of the robot;
$m$ and $\bm{J}$ denote the overall mass and moment of inertia, respectively;
$\bm{r}_{i}^{j} \in \mathbb{R}^3$ denotes the vector from the origin of the body frame system, 
the COM position of the payload, to the attachment position of the robot $j$ in quadrant $i$;
$\bm{e}_{3} = [0, 0, 1]^\top$ denotes the unit vector of the $z$-axis;
and $d_i^j \in \{1,-1\}$ denotes the rotational direction of the robot.
For $d_i^j=1$, the motion is clockwise, whereas for $d_i^j=-1$, it is counterclockwise.
Further, $c_{\mathrm{q}i}^j$ denotes the thrust torque conversion coefficient of the robot, 
and $g$ denotes the gravitational acceleration.
Notably, the Coriolis force is assumed to be zero because $\dot{\phi} \simeq 0$ and $\dot{\theta} \simeq 0$ are assumed during hovering.

From eq. (\ref{eq:dy}), a state-space equation can be represented as follows:
\begin{equation}
	\dot{\bm{X}} = \bm{A} \bm{X} + \bm{B}_n \bm{u}_n - \bm{w} ,
	\label{eq:st}
\end{equation}
where $\bm{X} = [\bm{x}^{\top}, \dot{\bm{x}}^{\top}, \bm{\phi}^{\top}, \dot{\bm{\phi}}^{\top}]^\top \in \mathbb{R}^{12}$, $\bm{x} = [x,y,z]^\top$ and $\bm{\phi} = [\phi, \theta, \psi]^\top$ denote the state, $\bm{u}_{n} = [u^1_1, \dots, u_4^{n_4}]^\top \in \mathbb{R}^{n}$ denotes the control input, and $n$ denotes number of total robots. 
Moreover, $\bm{A}$, $\bm{B}_n$, and $\bm{w}$ are defined as
\begin{equation*}
	\begin{split}
		\bm{A} &= \left[
			\begin{array}{cccc}
							\bm{0}_{3} & \bm{I}_{3} & \bm{0}_{3} & \bm{0}_{3} \\
							\bm{0}_{3} & \bm{0}_{3} & \bm{G}_{\mathrm{M}} & \bm{0}_{3} \\
							\bm{0}_{3} & \bm{0}_{3} & \bm{0}_{3} & \bm{I}_{3} \\
							\bm{0}_{3} & \bm{0}_{3} & \bm{0}_{3} & \bm{0}_{3} \\                        
			\end{array}
			\right]  \in \mathbb{R}^{12 \times 12} ,\\
		\bm{B}_n &= \left[
			\begin{array}{c}
						\bm{0}_{3 \times n} \\
						\bm{B}_{\mathrm{x}n} \\
						\bm{0}_{3 \times n} \\
						\bm{B}_{\mathrm{\Phi}n} \\                        
			\end{array}
			\right] \in \mathbb{R}^{12 \times n} ,\hspace{2mm}
		\bm{w} = \left[ 
			\begin{array}{c}
						\bm{0}_{3 \times 1} \\
						\bm{g}_{\mathrm{V}} \\
						\bm{0}_{3 \times 1} \\
						\bm{0}_{3 \times 1} \\                        
			\end{array}
			\right] \in \mathbb{R}^{12},
	\end{split}
\end{equation*}
where
\begin{equation*}
	\begin{split}
		\bm{G}_{\mathrm{M}} &= \left[
			\begin{array}{ccc}
				0 & g & 0 \\
				-g & 0 & 0 \\
				0 & 0 & 0 \\				
			\end{array}
		\right] , \hspace{2mm}
		\bm{g}_{\mathrm{V}} = \left[
			\begin{array}{c}
				0 \\
				0 \\
				g
			\end{array}
		\right] ,\\
		\bm{B}_{\mathrm{x} n} &= \frac{1}{m}[\bm{e}_{3}, \dots, \bm{e}_{3}] \in \mathbb{R}^{3 \times n} ,\\
		\bm{B}_{\mathrm{\Phi} n} &= \bm{J}^{-1} \left[
			\begin{array}{ccc}
				\left( \bm{r}_1^1 \times \bm{e}_{3} \right)_{12}, & \dots, & \left( \bm{r}_4^{n_4} \times \bm{e}_{3} \right)_{12} \\[4pt]
				d^1_1 c^1_{\mathrm{q}1}, & \dots, & d^{n_4}_4 c^{n_4}_{\mathrm{q}4} \\
			\end{array}
		\right]
		\in \mathbb{R}^{3 \times n}.
	\end{split}
\end{equation*}
In addition, $\bm{I}_{3}$ denotes the three-dimensional identity matrix 
and $\bm{0}_{3 \times 1}$, $\bm{0}_{3}$, and $\bm{0}_{3 \times n}$ denote zero matrices with three rows and one, three, and $n$ columns, respectively.

This study addresses fluctuations in mass and COM and robot failure.
From eq. (\ref{eq:st}), 
these fluctuations affect only the $\bm{B}_{\mathrm{x} n}$ and $\bm{B}_{\mathrm{\Phi} n}$ matrices. 
Thus, we focus on the $\bm{B}$ matrix for control robustness.

\section{Decentralized control}\label{sec_4}
We proposed a decentralized controller with two feedback controllers. 
The controller was designed following three main steps: first, an error system with four inputs and four outputs was derived from eq. (\ref{eq:st}).
Second, a robust feedback controller (RFC) renders the error system strictly positive real (SPR). 
Third, VG-ASSC ensures asymptotically stabilization to the reference of the system using only broadcasted error. 

\subsection{Error system}\label{sec_4-1}
In this section, we derive the error system from eq. (\ref{eq:st}) using input and output transformations.

\subsubsection{Input transformation}\label{sec_4-1-1}
Notably, the input matrix $\bm{B}_n$ should be a full rank for the system to be SPR.
Eq. (\ref{eq:st}) is converted to a four-input equation because the rank of the input matrix is four for a flight system based on rigidly connected rotors such as a multi-copter.
Therefore, we make the following assumptions:

\begin{itemize}
	\item The robot $j$ $(=1,...,n_i)$ in quadrant $i$ locates the representative point of quadrant $i$. Let $\bm{r}_i$ be the vector from the COM to the representative point in quadrant $i$; the vector of each robot in quadrant $i$ from the COM can be expressed as $\bm{r}_i$ ($\bm{r}_i \simeq \bm{r}^{1}_{i} \simeq \dots \simeq \bm{r}_i^{n_i}$).
	\item The rotors of robots in quadrant $i$ rotate in the same direction. Let $d_i$ be the rotation direction of the robots in quadrant $i$; the rotation direction of each robot in quadrant $i$ can be expressed as $d_i$ ($d_i = d_i^1 = \dots = d_i^{n_i}$).
	\item Among the $d_i$, at least one positive and one negative values are obtained.
	\item The average thrust--torque coefficients are the same for every quadrant. Let $c_{\mathrm{q}}$ be the average of the thrust--torque coefficient; the thrust--torque coefficient of each robot is $c_{\mathrm{q}}$ ($c_{\mathrm{q}} \simeq \frac{1}{n_1} \sum_{j=1}^{n_1} c_{\mathrm{q}1}^j \simeq \dots \simeq \frac{1}{n_4} \sum_{j=1}^{n_4} c_{\mathrm{q}4}^j$).
\end{itemize}

Under these assumptions, eq. (\ref{eq:st}) can be transformed into a four-input equation as follows:
\begin{equation}
	\dot{\bm{X}} = \bm{A} \bm{X} + \bm{B} \bm{U} - \bm{w}
	\label{eq:st4}
\end{equation}
where $\bm{U} =\left[U_1, U_2, U_3, U_4 \right]^{\top}$, $U_i = \sum_{j=1}^{n_i} u_i^j$, 
\begin{equation*}
	\begin{split}
	\bm{B} &= \left[
		\begin{array}{c}
			\bm{0}_{3 \times 4} \\
			\bm{B}_{\mathrm{X}} \\
			\bm{0}_{3 \times 4} \\
			\bm{B}_{\mathrm{\Phi}} \\
		\end{array}
	\right] \in \mathbb{R}^{12 \times 4} , \\
	\bm{B}_{\mathrm{X}} &= \frac{1}{m} \left[\bm{e}_{3}, \bm{e}_{3}, \bm{e}_{3}, \bm{e}_{3}\right] , \\
 	\bm{B}_{\mathrm{\Phi}} &= \bm{J}^{-1} \times \\
	& \left[
		\begin{array}{cccc}
			(\bm{r}_1 \times \bm{e}_{3})_{12} \!\!\!\!  & (\bm{r}_2 \times \bm{e}_{3})_{12} \!\!\!\! & (\bm{r}_3 \times \bm{e}_{3})_{12} \!\!\!\! &  (\bm{r}_4 \times \bm{e}_{3})_{12} \\
			d_1 c_{\mathrm{q}} \!\!\!\!  & d_2 c_{\mathrm{q}} \!\!\!\!  & d_3 c_{\mathrm{q}} \!\!\!\! & d_4 c_{\mathrm{q}} \\
		\end{array}
	\right].
	\end{split}
\end{equation*}
$\bm{0}_{3 \times 4}$ denotes the zero matrix with three rows and four columns.

\subsubsection{Error dynamics}\label{sec_4-1-2}
We derive an error system of eq. (\ref{eq:st4}) for the reference expressed as
\begin{equation}
	\dot{\bm{X}}_{\mathrm{r}} = \bm{A} \bm{X}_{\mathrm{r}} + \bm{B} \bm{U}_{\mathrm{r}} - \bm{w},
	\label{eq:ideal_d}
\end{equation}
where $\bm{U}_{\mathrm{r}} \in \mathbb{R}^{4}$ is the reference input, $\bm{X}_{\mathrm{r}} = [\bm{x}_\mathrm{r}^{\top}, \dot{\bm{x}}_\mathrm{r}^{\top}, \bm{\phi}_{\mathrm{r}}^{\top}, \dot{\bm{\phi}}_{\mathrm{r}}^{\top}]^{\top}$ $\in \mathbb{R}^{12}$ and $\bm{x}_{\mathrm{r}}$ and $\bm{\phi}_{\mathrm{r}}$ are the reference position and angle, respectively.
$\bm{X}_{\mathrm{r}}$ is generated under the condition $\dot{\bm{X}}_{\mathrm{r}} = 0$ to realize hovering.
Notably, $\bm{U}_\mathrm{r}$ includes thrusts for hovering and  compensating minuscule positional errors.
$\bm{U}_\mathrm{r}$ is assumed to exists such that the thrust constraints are satisfied. 
This assumption is reasonable, as the system cannot operate beyond its physical limits.

From eqs. (\ref{eq:st4}) and (\ref{eq:ideal_d}), the error dynamics can be expressed as
\begin{equation}
	\bm{\dot{\xi}} = \bm{A} \bm{\xi} + \bm{B} (\bm{U} - \bm{U}_{\mathrm{r}}),
	\label{eq:error}
\end{equation}
where 
$\bm{\xi} = [\bm{x}_{\mathrm{e}}^{\top}, \bm{\dot{x}}_{\mathrm{e}}^{\top}, \bm{\phi}_{\mathrm{e}}^{\top}, \bm{\dot{\phi}}_{\mathrm{e}}^{\top}]^\top$, $\bm{x}_{\mathrm{e}} = \bm{x} - \bm{x}_\mathrm{r} = \left[x_{\rm e}, y_{\rm e}, z_{\rm e} \right]^{\top}$, and $\bm{\phi}_{\mathrm{e}} = \bm{\phi} - \bm{\phi}_{\mathrm{r}} = \left[\varphi_{\rm e}, \theta_{\rm e}, \psi_{\rm e} \right]^{\top}$.

\subsubsection{Output transformation}\label{sec_4-1-3}
The dimensions of the input and output are the same to render the system SPR.
Therefore, we consider the following four outputs:
\begin{equation}
	\left[ 
		\begin{array}{c}
			x_{\mathrm{e}} \\
			y_{\mathrm{e}} \\
			z_{\mathrm{e}} \\
			\psi_{\mathrm{e}} \\
		\end{array}
	\right]		
		= \bm{C} \bm{\xi},
	\label{eq:out0}
\end{equation}
where $\bm{C} \in \mathbb{R}^{4 \times 12}$ is the output matrix.
If the system composed of eqs. (\ref{eq:error}) and (\ref{eq:out0}) is SPR, a positive definite symmetric matrix $\bm{Q}$ exists satisfying the following equation \cite{van2000}:
\begin{equation}
	\bm{Q}\bm{B} = \bm{C}^\top.
	\label{eq:out01}
\end{equation}
This shows that $\bm{C}$ depends on $\bm{B}$.
However, in this study, $\bm{B}$ is unknown and fluctuates within a certain range.

Therefore, using the differential values of eq. (\ref{eq:out0}) yields the following four outputs with feedthrough terms: 
\begin{equation}
	\begin{split}
		&x_{\mathrm{p}} = c_{\mathrm{x}4} \ddddot{x}_{\mathrm{e}} + c_{\mathrm{x}3} \dddot{x}_{\mathrm{e}} + c_{\mathrm{x}2} {\ddot{x}_{\mathrm{e}}} + c_{\mathrm{x}1} \dot{x}_{\mathrm{e}} + c_{\mathrm{x}0} x_{\mathrm{e}} ,\\
		&y_{\mathrm{p}} = c_{\mathrm{y}4} \ddddot{y}_{\mathrm{e}} + c_{\mathrm{y}3} \dddot{y}_{\mathrm{e}} + c_{\mathrm{y}2} {\ddot{y}_{\mathrm{e}}} + c_{\mathrm{y}1} \dot{y}_{\mathrm{e}} + c_{\mathrm{y}0} y_{\mathrm{e}} ,\\
		&z_{\mathrm{p}} =  c_{\mathrm{z}2} \ddot{z}_{\mathrm{e}} + c_{\mathrm{z}1} \dot{z}_{\mathrm{e}} + c_{\mathrm{z}0} z_{\mathrm{e}} ,\\
		&\psi_{\mathrm{p}} = c_{\mathrm{\psi} 2} \ddot{\psi}_{\mathrm{e}} + c_{\mathrm{\psi} 1} \dot{\psi}_{\mathrm{e}} + c_{\mathrm{\psi} 0} \psi_{\mathrm{e}} .\\
	\end{split}
	\label{eq:out1}
\end{equation}
From eq. (\ref{eq:error}), eq. (\ref{eq:out1}) can be expressed as: 
\begin{equation}
	\begin{split}
		\bm{\zeta} &= \bm{C} \bm{\xi} + \bm{D} (\bm{U} - \bm{U}_{\mathrm{r}}), \\
	\end{split}
	\label{eq:out2}
\end{equation}
where $\bm{\zeta} = \left[x_{\mathrm{p}}, y_{\mathrm{p}}, z_{\mathrm{p}}, \psi_{\mathrm{p}}\right]^{\top}$, $\bm{C} = [\bm{C}_0, \bm{C}_1, \bm{C}_2, \bm{C}_3] \in \mathbb{R}^{4 \times 12}$,
\begin{equation*}
	\begin{split}
		\bm{C}_0 &= \left[
			\begin{array}{ccc}
				c_{\mathrm{x}0} & 0 & 0 \\
				0 & c_{\mathrm{y}0} &  0\\
				0 & 0 & c_{\mathrm{z}0} \\
				0 & 0 & 0 \\
			\end{array}
		\right], 
	 \hspace{3.4mm} \bm{C}_1 = \left[
			\begin{array}{ccc}
				c_{\mathrm{x}1} & 0 & 0 \\
				0 & c_{\mathrm{y}1} &  0\\
				0 & 0 & c_{\mathrm{z}1} \\
				0 & 0 & 0 \\
			\end{array}
		\right], \\
		 \bm{C}_2 &= \left[
			\begin{array}{ccc}
				0 \!\!\!\! & gc_{\mathrm{x}2} \!\!\!\! & 0 \\
				-g c_{\mathrm{y}2} \!\!\!\! & 0 \!\!\!\! &  0\\
				0 \!\!\!\! & 0 \!\!\!\! & 0 \\
				0 \!\!\!\! & 0 \!\!\!\! & c_{\mathrm{\psi}0} \\
			\end{array}
		\right], 
		\bm{C}_3 = \left[
			\begin{array}{ccc}
				0 \!\!\!\! & gc_{\mathrm{x}3} \!\!\!\! & 0 \\
				-g c_{\mathrm{y}3} \!\!\!\! & 0 \!\!\!\! &  0\\
				0 \!\!\!\! & 0 \!\!\!\! & 0 \\
				0 \!\!\!\! & 0 \!\!\!\! & c_{\mathrm{\psi}1} \\
			\end{array}
		\right], \\
		\bm{D} &= \left[
			\begin{array}{cccc}
				0 & 0 & gc_{\mathrm{x}4} & 0 \\
				0 & -g c_{\mathrm{y}4} & 0 & 0 \\
				c_{\mathrm{z}2} & 0 & 0 & 0 \\
				0 & 0 & 0 & c_{\mathrm{\psi}2} \\
			\end{array}
		\right]
		\left[
			\begin{array}{c}
				\frac{\bm{1}_4^\top}{m} \\[5pt]
				\bm{B}_{\mathrm{\Phi}}  \\
			\end{array}
		\right] \in \mathbb{R}^{4 \times 4}, \\
		\bm{1}_4^\top &= [1,1,1,1] . \\
	\end{split}
\end{equation*}
By substituting eq. (\ref{eq:out2}) for eq. (\ref{eq:out0}), the condition for the system comprising eqs. (\ref{eq:error}) and (\ref{eq:out2}) to be SPR is an inequality.
The second and higher order derivatives of $x_\mathrm{e}$ and $y_\mathrm{e}$ in eq. (\ref{eq:out1}) are calculated from $\ddot{x}_{\mathrm{e}} = g \theta$ and $\ddot{y}_{\mathrm{e}} = -g \phi$ within eq. (\ref{eq:error}). 

In the following sections,
RFC and VG-ASSC are proposed to design $\bm{U}_{\rm f} = [U_{{\rm f}1}, U_{{\rm f}2}, U_{{\rm f}3}, U_{{\rm f}4}]^{\top}$ and $\bm{U}_{\rm s} = [U_{\rm s1}, U_{\rm s2}, U_{\rm s3}, U_{\rm s4}]^{\top}$, respectively.
As a result, the input $\bm{U}$ to the system with eqs. (\ref{eq:error}) and (\ref{eq:out2}) is the sum of these inputs:

\begin{equation}
	\bm{U} = \bm{U}_{\mathrm f} + \bm{U}_{\mathrm s} .
	\label{eq:U_sum}
\end{equation}

\subsection{Robust feedback controller (RFC)}\label{sec_4-1-5}
In this section, based on the shared fluctuation ranges, 
we design $\bm{U}_{\mathrm f}$ of RFC in eq. (\ref{eq:U_sum}) without communications among robots to ensure that the system satisfies eqs. (\ref{eq:error}) and (\ref{eq:out2}) SPR.

\begin{figure}[t]
	\centering
		\includegraphics[keepaspectratio, scale=0.4]{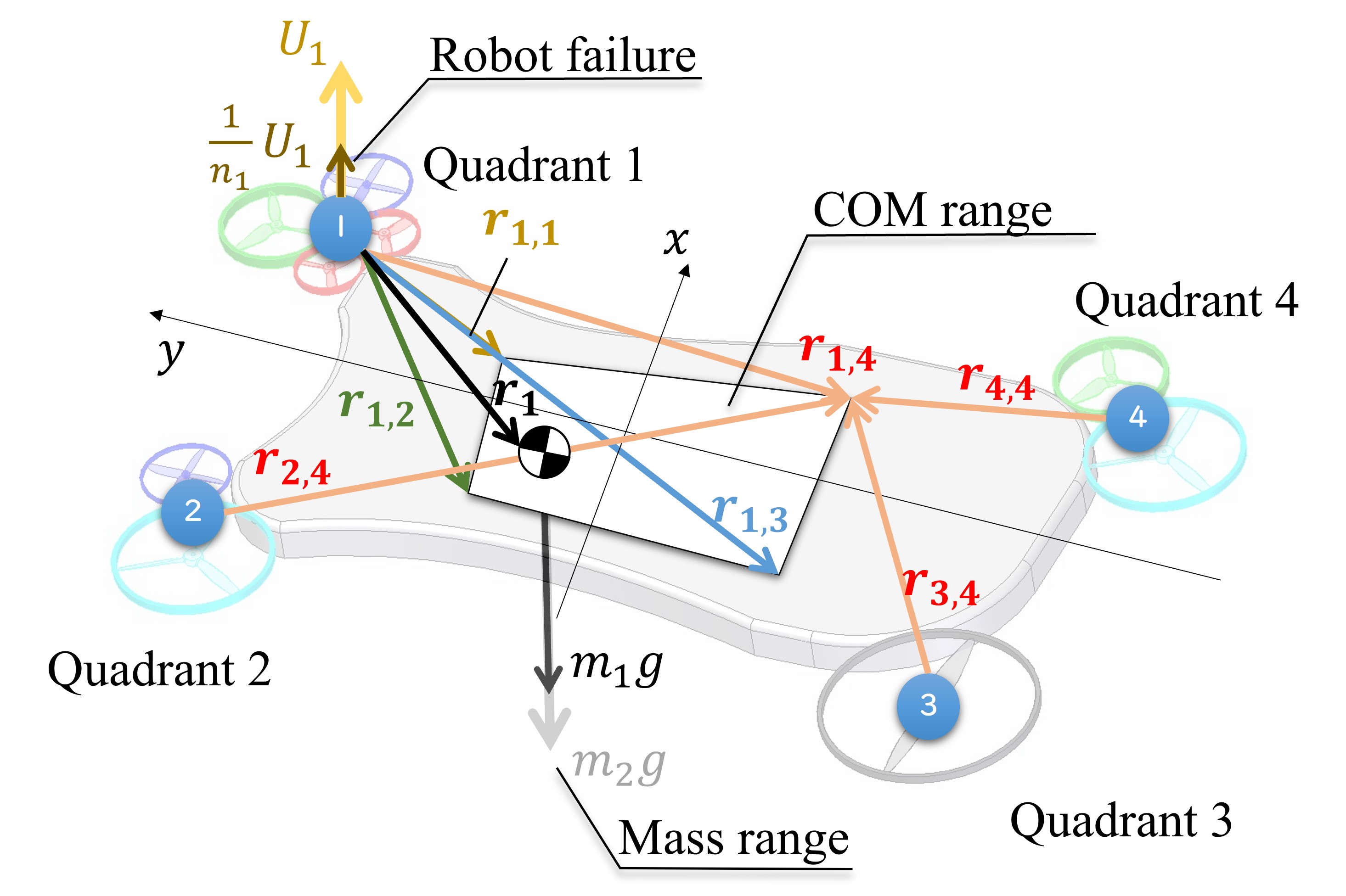}
	  \caption{Fluctuations}\label{fig:pory}
\end{figure}

\subsubsection{Uncertainty sharing}\label{sec_4-1-4}
In the RFC, the fluctuation ranges of mass, COM, and robot failure are shared, as shown in Fig. \ref{fig:pory}.
The fluctuations are treated as a change of $\bm{B}$ in eq. (\ref{eq:error}).
Furthermore, the thrust loss per robot failure in quadrant $i$ is assumed to be $\frac{1}{n_i} U_i $.
Therefore, the uncertainties associated with the mass, COM, and robot failures are defined as follows:
\begin{equation*}
	\begin{split}
		\frac{1}{m} &= \sum_{h=1}^2 \frac{\alpha_h}{m_h}, \\
		\bm{r_i} &= \sum_{l=1}^4 \beta_l \bm{r_{i,l}} , \\
		U_i &\rightarrow \left( \sum_{k_i = 1}^2 \gamma_{i,k_i} b_{i, k_i} \right) U_i ,
	\end{split}
\end{equation*}
where $\bm{r}_{i,l}$ is the vector from the representative point in quadrant $i$ to the vertex of the COM range, $m_h$ within the payload is assumed to be in $[m_1, m_2]$, $b_{i,1} = \frac{1}{n_i}$, and $b_{i,2} = 1$.
Moreover, $\alpha_h$, $\beta_l$, $\gamma_{1,k_1}$, $\gamma_{2,k_2}$, $\gamma_{3,k_3}$, and $\gamma_{4,k_4}$ are satisfied with 
\begin{equation*}
	\begin{split}
	\sum_{h=1}^2 \alpha_h &= \sum_{l=1}^4 \beta_l = \sum_{k_1 = 1}^2 \gamma_{1,k_1} = \sum_{k_2 = 1}^2 \gamma_{2,k_2} \\ 
	                      &= \sum_{k_3 = 1}^2 \gamma_{3,k_3} = \sum_{k_4 = 1}^2 \gamma_{4,k_4} = 1 .
	\end{split}
\end{equation*}
As a result, $\bm{B}$ with uncertainties due to the fluctuations are expressed as 
\begin{equation*}
	\begin{split}
		\bm{B} = \sum_{h=1}^2 \sum_{l=1}^4 & \sum_{k_1 = 1}^2 \sum_{k_2 = 1}^2 \sum_{k_3 = 1}^2 \sum_{k_4 = 1}^2 \\
		& \alpha_h \beta_l \gamma_{1,k_1} \gamma_{2,k_2} \gamma_{3,k_3} \gamma_{4,k_4}
		\bm{B}_{h, l, k_1, k_2, k_3, k_4},
	\end{split}
\end{equation*}
where
\begin{equation*}
	\begin{split}
		& \bm{B}_{h, l, k_1, k_2, k_3, k_4} = \left[
			\begin{array}{c}
				\bm{0}_{3 \times 4} \\
				\bm{B}_{\mathrm{X} h} \\
				\bm{0}_{3 \times 4} \\
				\bm{B}_{\mathrm{\Phi} l} \\
			\end{array}
		\right]
		{\rm diag} \left( b_{1,k_1}, b_{2, k_2}, b_{3, k_3}, b_{4, k_4} \right) , \\
		&\bm{B}_{\mathrm{X}h} = \frac{1}{m_h} \left[
			\begin{array}{cccc}
				\bm{e}_{3} & \bm{e}_{3} & \bm{e}_{3} & \bm{e}_{3} \\
			\end{array}
		\right] \in \mathbb{R}^{3 \times 4} ,\\
		&\bm{B}_{\mathrm{\Phi} l} = \bm{J}^{-1} \times\\
		&\left[
			\begin{array}{cccc}
				(\bm{r}_{1,l} \times \bm{e}_{3})_{12} \! \! & (\bm{r}_{2,l} \times \bm{e}_{3})_{12} \! \! & (\bm{r}_{3,l} \times \bm{e}_{3})_{12} \! \! & (\bm{r}_{4,l} \times \bm{e}_{3})_{12} \\
			d_1 c_{\mathrm{q}} \! \! & d_2 c_{\mathrm{q}} \! \! & d_3 c_{\mathrm{q}} \! \! & d_4 c_{\mathrm{q}} \\
			\end{array}
		\right] \: \: \in \mathbb{R}^{3 \times 4} .
	\end{split}
\end{equation*}
In this study, the effect of inertia fluctuations is assumed to be negligible.

\begin{figure}[t]
	\centering
		\includegraphics[keepaspectratio, scale=0.7]{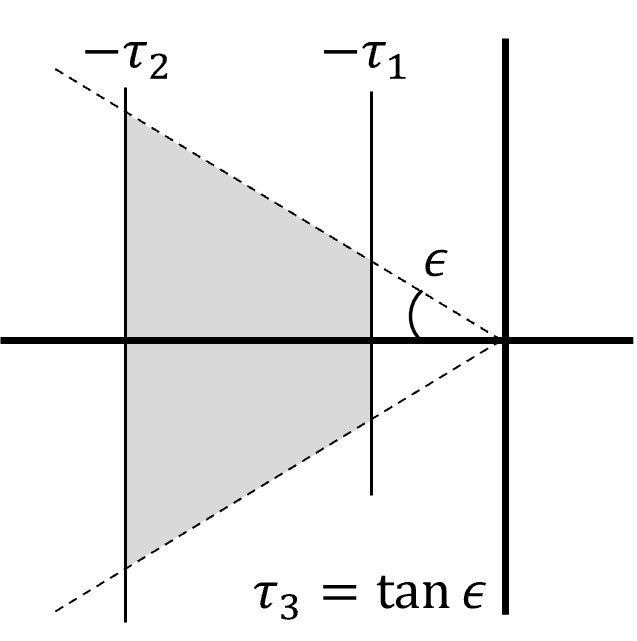}
	  \caption{Pole assignment}\label{fig:lmi_c}
\end{figure}

\subsubsection{Strictly positive realization}\label{sec_4-2-2-xx}
First, to increase the flexibility in the RFC design, $\bm{G} \in \mathbb{R}^{4 \times 4}$ is introduced for eq. (\ref{eq:out2}), and the new output equation is expressed as follows:
\begin{equation}
	\begin{split}
		\bm{\eta} &= \bm{G}\bm{\zeta} \\
	  & = \bm{G} \bm{C} \bm{\xi} + \bm{G} \bm{D} \left(\bm{U} - \bm{U}_{\rm r}\right) .
	\end{split}
\label{eq:out3}
\end{equation}
Substituting $\bm{U}_{\rm f} = -\bm{F} \bm{\xi}$ into eqs. (\ref{eq:error}), (\ref{eq:out2}), and (\ref{eq:out3}), the error system is described as
\begin{equation}
	\begin{split}
		\dot{\bm{\xi}} &= (\bm{A}- \bm{B} \bm{F}) \bm{\xi} + \bm{B} (\bm{U}_{\rm s} - \bm{U}_{\mathrm{r}}) ,\\
		\bm{\eta} &= \bm{G} (\bm{C} - \bm{D} \bm{F}) \bm{\xi} + \bm{G} \bm{D} (\bm{U}_{\rm s} - \bm{U}_{\mathrm{r}}).
	\end{split}
	\label{eq:st_fin}
\end{equation}

Second, 
for the aforementioned system to be SPR, it is sufficient that a $\bm{F}$, $\bm{G}$, and positive definite symmetric matrix $\bm{P}$ exist satisfying the following conditions:
\begin{equation}
	\begin{split}
	&\bm{P}(\bm{A}  -  \bm{B} \bm{F})  +  (\bm{A}  -  \bm{B} \bm{F})^\top \bm{P}  +  \\
	& \quad \left(\bm{P} \bm{B} -  \left(\bm{G} \left(\bm{C}-\bm{D}\bm{F}\right)\right)^\top \right)(\bm{G} \bm{D}  + (\bm{G} \bm{D})^\top)^{-  1} \times \\
	& \quad \quad \left(\bm{P} \bm{B}  -  \left(\bm{G} \left(\bm{C}-\bm{D}\bm{F}\right) \right)^\top \right)^\top  \prec  0,
	\end{split}
\label{eq:lmi1}
\end{equation}
\begin{equation}
	\bm{G} \bm{D}+(\bm{G} \bm{D})^\top \succ  0.
\label{eq:lmi1-1}
\end{equation}
However, the conditions (\ref{eq:lmi1}) and (\ref{eq:lmi1-1}) cannot be solved directly by LMI \cite{van2000}.
Therefore, the conditions (\ref{eq:lmi1}) and (\ref{eq:lmi1-1}) are converted into the following.
\begin{equation}
	\left[ \!
					\begin{array}{cc}
					 \bm{A} \bm{Q} \! + \! \bm{Q} \bm{A}^\top \! \! - \! \bm{B} \bm{R} \! -  \! \bm{R}^\top  \bm{B} & \bm{Q} \bm{C}^\top \! \! - \! \bm{B} \bm{S}^\top \! + \! \bm{R}^{\top} \bm{D}^{\top} \\
					 \bm{C} \bm{Q} \! - \! \bm{S}  \bm{B}^\top \! + \! \bm{D}\bm{R}   & -\bm{D} \bm{S}^\top \! \! - \! \bm{S} \bm{D}^\top \\
					\end{array}
		 \!\right] \prec 0,
		 \label{eq:lmi2}
\end{equation}
where $\bm{Q}(= \bm{P}^{-1}) \in \mathbb{R}^{12 \times 12}$, $\bm{R}(=\bm{F} \bm{Q}) \in \mathbb{R}^{4 \times 12}$, and
$\bm{S}(=\bm{G}^{-1}) \in \mathbb{R}^{4 \times 4}$.

Third, in this study, pole and gain constraints are considered in the RFC design.
Figure \ref{fig:lmi_c} shows the pole range to be assigned. 
The pole constraints are expressed as follows:  
\begin{equation}
	\bm{A} \bm{Q} \! + \! \bm{Q} \bm{A}^\top \! \! - \! \bm{B} \bm{R} \! -  \! \bm{R}^\top  \bm{B} + 2\tau_1 \bm{Q} \prec 0,
		 \label{eq:lmi_c1}
\end{equation}
\begin{equation}
	\bm{A} \bm{Q} \! + \! \bm{Q} \bm{A}^\top \! \! - \! \bm{B} \bm{R} \! -  \! \bm{R}^\top  \bm{B} + 2\tau_2 \bm{Q} \succ 0,
		 \label{eq:lmi_c2}
\end{equation}
\begin{equation}
	\left[ \!
					\begin{array}{cc}
					 \tau_3 \left(\bm{A} \bm{Q} + \bm{Q} \bm{A}^\top \right) & \bm{A} \bm{Q} - \bm{Q} \bm{A}^\top \\
					 \bm{Q} \bm{A}^\top - \bm{A} \bm{Q}   & \tau_3 \left(\bm{A} \bm{Q} + \bm{Q} \bm{A}^\top  \right) \\
					\end{array}
		 \!\right] \prec 0.
		 \label{eq:lmi_c3}
\end{equation}
To design the gain $\bm{F}$ satisfying eq.~(\ref{eq:lmi2}) to the smallest possible value and reduce the feasible solution space for improved tractability, the following gain constraint is considered
\begin{equation}
	\begin{split}
		& \min \kappa \\
		& {\rm s.t.} \\
		& \bm{Q} \succ \bm{I} \\
		& \left[ \!
					\begin{array}{cc}
					 \kappa^2 & \bm{R}^\top \\
					 \bm{R}   & \bm{I} \\
					\end{array}
		\!\right] \succ 0 .\\
	\end{split}
	\label{eq:lmi_c4}
\end{equation}
This utilizes Schur's lemma and is equivalent to the following constraint.
\begin{equation*}
	\| \bm{R} \|_2 < \kappa,
\end{equation*}
where $\| \cdot \|_2$ denotes the $\ell_2$ norm.

Finally, for decentralized control, to enable each robot to design its $\bm{U}_{\rm f}$, the $\bm{B}$ matrix used for the LMI extends to $\bm{B} \:{\rm diag}(n_1,n_2,n_3,n_4)$.
Then, the control input of the RFC of the $j$-th robot in quadrant $i$ is as follows:
\begin{equation}
	u_{\mathrm{f}i}^{j} = \bm{U}_{{\rm f}i} .
	\label{eq:uf_dec}
\end{equation}
%
\subsection{VG-ASSC for multiple output systems}\label{sec_4-2}
We designed $u_{{\rm s}i}^{j}$ in VG-ASSC.
From eqs. (\ref{eq:U_sum}) and (\ref{eq:uf_dec}), the control input of the $j$-th robot in quadrant $i$ is as follows:
\begin{equation*}
	u_i^j = u_{\mathrm{f}i}^{j} + u_{{\rm s}i}^j .
\end{equation*}
The complete structure of the proposed controller that design this control input is shown in Fig. \ref{fig:cont}.

To ensure the asymptotic stability of an error, the VG-ASSC \cite{Amano} for single output systems was adopted. 
However, to realize flight, multiple output errors must be controlled.  
In this study, each single-rotor robot uses one element of mixed error $\bm{\eta}$ as a feedback signal to its VG-ASSC.
Eq. (\ref{eq:out2}) includes the control input $\bm{U}$ in addition to the error $\bm{\xi}$.
Therefore, the VG-ASSC uses the approximated output $\bm{\eta}$. 



\begin{figure}[t]
	\centering
	\begin{minipage}[b]{0.55\linewidth}
		\centering
		\includegraphics[keepaspectratio, scale=0.4]{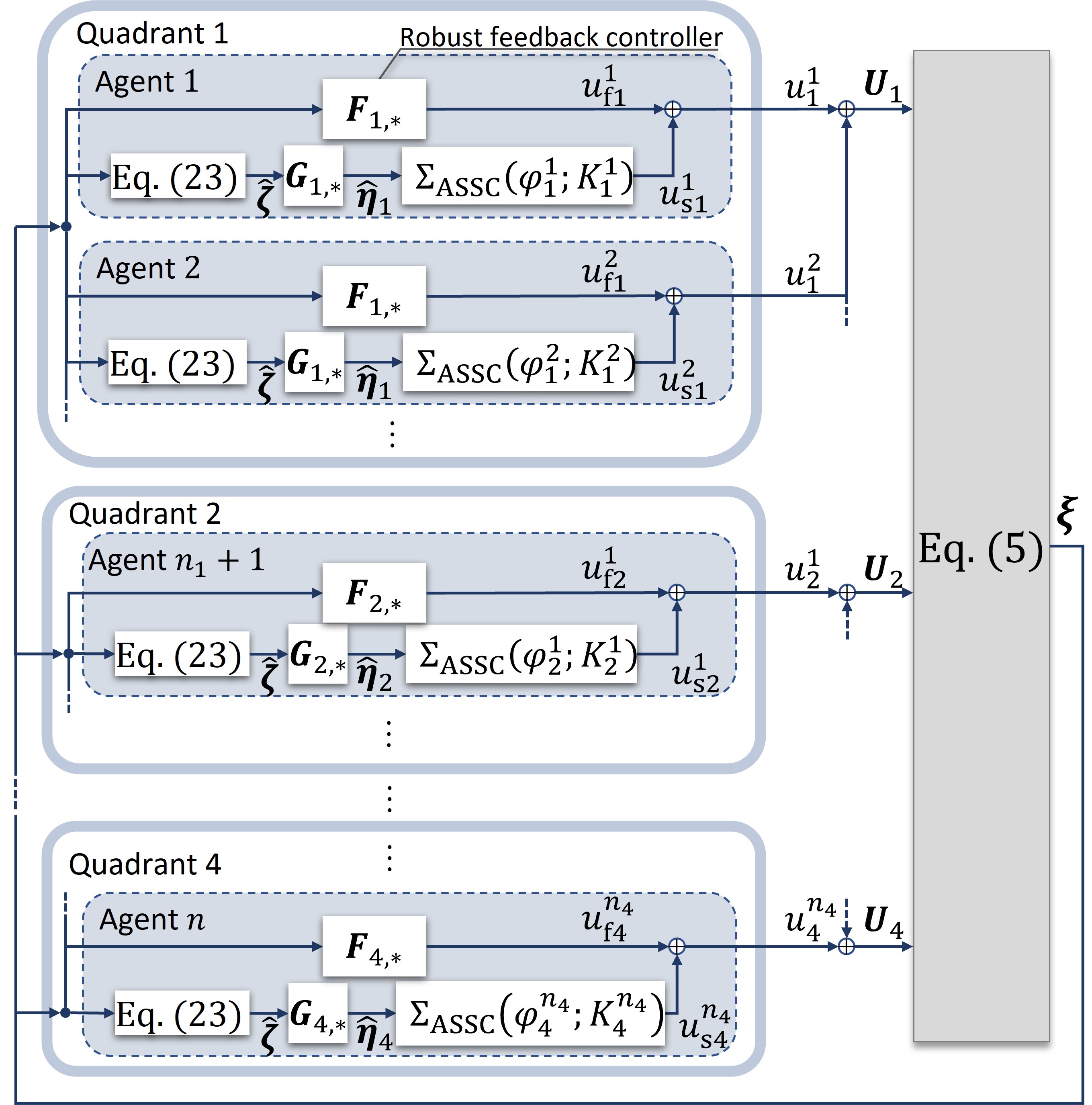}
		\caption{Proposed controller}\label{fig:cont}
	\end{minipage}
	\hfill
	\begin{minipage}[b]{0.35\linewidth}
		\centering
		\includegraphics[keepaspectratio, scale=0.45]{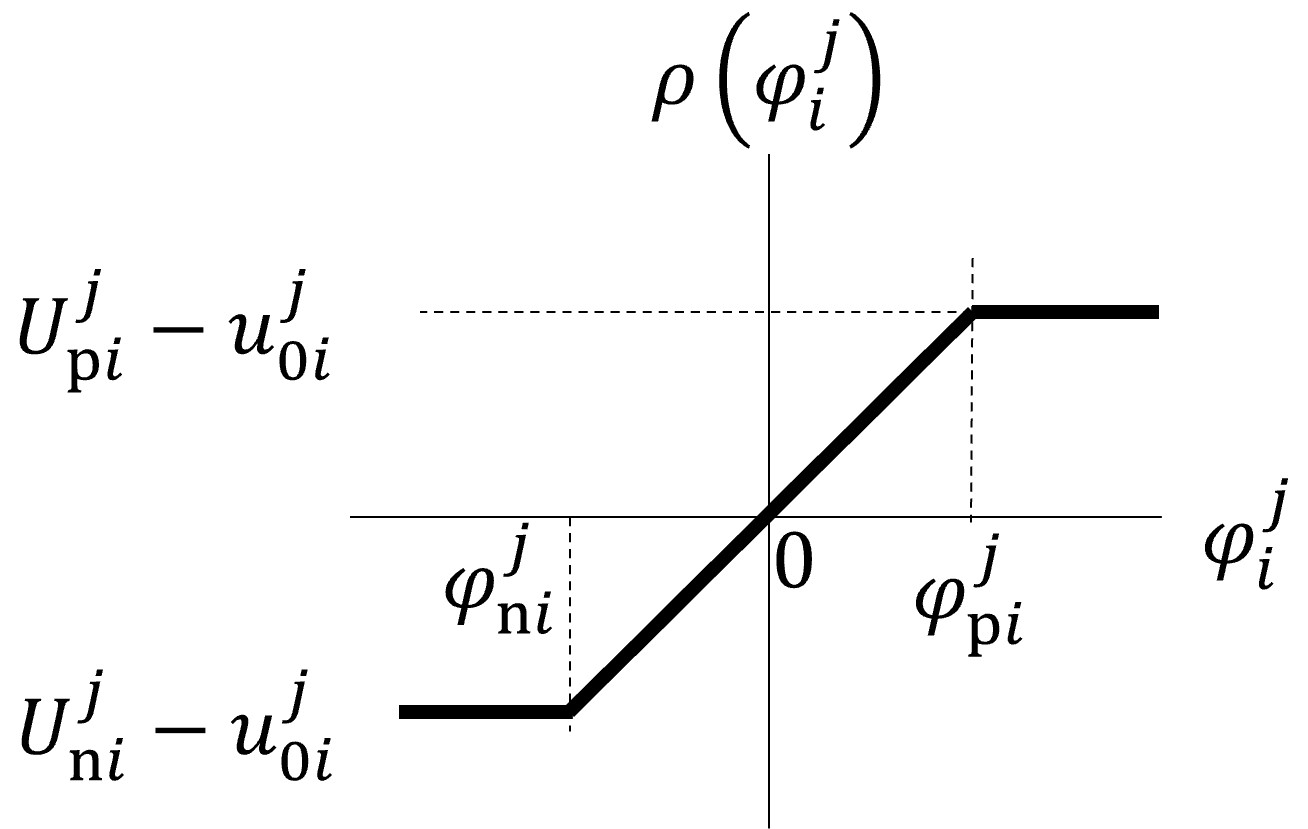}
		\caption{Switching function $\rho$ \cite{Oishi2021ICRA}}\label{fig:us}
	\end{minipage}
\end{figure}

\subsubsection{Asymptotically stabilization}\label{sec_4-2-2}
The VG-ASSC control is based on the $\bm{\eta}$ of eq. (\ref{eq:out3}). 
The $j$-th robot in quadrant $i$ uses the output $\eta_i^j (= \eta_i)$.
The VG-ASSC featuring the use of $\eta_i^j$ is as follows:
\begin{equation}
     \begin{split}
     \Sigma_{\rm{ASSC}}& \left \{
             \begin{array}{l}
     \dot{\varphi}_{i}^{j} = -K_{i}^{j}(\varphi_i^{j},\eta_i^j) \eta_i^j\\[3pt]
     u_{\mathrm{s}i}^j = \rho(\varphi_{i}^{j}) + u_{0i}^{j} 
             \end{array} \right. ,\\[5pt]
     \rho(\varphi_{i}^{j}) =
     & \left \{
     \begin{array}{l}
        U_{\mathrm{p}i}^{j} - u_{0i}^j : \varphi_{i}^{j} > \varphi_{\mathrm{p}i}^{j}  \\[4pt]
				\frac{U_{\mathrm{p}i}^j - U_{\mathrm{n}i}^j }{\varphi_{\mathrm{p}i}^{j}-\varphi_{\mathrm{n}i}^{j}}
				 \varphi_{i}^{j} : \varphi_{\mathrm{n}i}^j < \varphi_{i}^{j} \leq \varphi_{\mathrm{p}i}^{j}  \\[5pt]
     		U_{\mathrm{n}i}^{j} - u_{0i}^j: \varphi_{i}^{j} \leq \varphi_{\mathrm{n}i}^j
     \end{array} \right. ,
     \end{split}
     \label{eq:assc}
\end{equation}
where $K_{i}^j > 0 $ are the gains;
$U_{\mathrm{p}i}^j$ and $U_{\mathrm{n}i}^j$ are the upper and lower limits of the thrust applied by the robot, respectively;
$\rho$ is the switching function of $\varphi_{i}^{j}$, as shown in Fig. \ref{fig:us};
$\varphi_{\mathrm{p}i}^{j}$ is $\varphi_{i}^{j}$ with $(U_{\mathrm{p}i}^{j}-U_{\mathrm{n}i}^{j})/(\varphi_{\mathrm{p}i}^j - \varphi_{\mathrm{n}i}^j) \varphi_{i}^{j} = U_{\mathrm{p}i}^{j} - u_{0i}^j$;
$\varphi_{\mathrm{n}i}^{j}$ is $\varphi_{i}^{j}$ with $(U_{\mathrm{p}i}^{j}-U_{\mathrm{n}i}^{j})/(\varphi_{\mathrm{p}i}^j - \varphi_{\mathrm{n}i}^j) \varphi_{i}^{j} = U_{\mathrm{n}i}^{j} - u_{0i}^j$,
and $u_{0i}^j$ is the thrust for hovering.
Moreover, $K_i^j$ can be expressed as follows:
\begin{equation}
        \begin{split}
                 K_i^j(\varphi_i^j,\eta_i^j) =
                \left\{
                \begin{array}{l}
                \overline{k_i}^j : \varphi_{i}^j \eta_i^j > 0  \\[4pt]
                \underline{k_i}^j : \varphi_{i}^j \eta_i^j \le 0 
                \end{array} \right. \\
        \end{split}
				\label{eq:sw_gain}
\end{equation}
where $\overline{k_i}^j$ and $\underline{k_i}^j$ are positive values satisfying $0 < \underline{k_i}^j \leq \overline{k_i}^j$.
\begin{theorem}
	\label{thm:assc_a}
	Applying the controller of eq. (\ref{eq:assc}) to the error system comprising eq. (\ref{eq:st_fin}) shows that
	the error $\bm{\eta}$ satisfies $\bm{\eta} \rightarrow 0$ as $t \rightarrow \infty$.
\end{theorem}
\begin{proof}
	See Appendix \ref{sec_proof1} for the proof.
\end{proof}
This controller adds multiple outputs and offsets $u_{0i}^j$ to the controller proposed by Amano et al. \cite{Amano}.
Multiple outputs and offsets are necessary to manage aerial systems.
In addition, an accurate value of $u_{0i}^j$ must be obtained to effectively use the switching gain as defined in eq.\eqref{eq:sw_gain}.
In this control approach, $u_{0i}^j$ is updated the accurate value after the equilibrium state.

High-response decentralized control can be expected by the proposed controller.
$\varphi_i^j$ is 0 at ideal equilibrium for $u_{0i}^j$.
$u_{\mathrm{s}i}^j$ applies thrust in the opposite direction of the integral of ${\eta}_i^j$.
Thus, eq. (\ref{eq:assc}) temporarily yields the control input in the opposite direction to the desired control input when $\varphi_{i}^j {\eta}_i^j$ is positive.
Therefore, the response can be improved by providing a considerable gain when $\varphi_{i}^j {\eta}_i^j$ is positive.

\subsubsection{Output approximation}\label{sec_4-2-1}
Each robot requires the control input $U$ of other robots because $\bm{\eta}$ in eq.~(\ref{eq:st_fin}) includes $U$. 
To achieve decentralized control, we approximate $U$ to exclude the control input from other robots.
To approximate $\bm{\eta}$, we focus on the accelerations of $z_{\mathrm{e}}$ and $\bm{\phi}_{\mathrm{e}}$: 
\begin{equation*}
	\left[
		\begin{array}{c}
			\ddot{z}_{\mathrm{e}} \\
			\ddot{\bm{\phi}}_{\mathrm{e}} \\
		\end{array}
	\right] =
	\left[
		\begin{array}{c}
			\frac{\bm{1}_4^\top}{m} \\[5pt]
			\bm{B}_{\mathrm{\Phi}}  \\
		\end{array}
	\right]
	\left( \bm{U} - \bm{U}_{\mathrm{r}} \right) .
\end{equation*}
By solving this equation for $\bm{U} - \bm{U}_{\mathrm{r}}$ and substituting it in eq. (\ref{eq:out2}), the approximated $\hat{\bm{\zeta}}$ is obtained as 
\begin{equation}
	\hat{\bm{\zeta}} = \bm{C} \bm{\xi} + \bm{\hat{D}}
	\left[
		\begin{array}{c}
			\ddot{z}_{\mathrm{e}} \\
			\ddot{\bm{\phi}}_{\mathrm{e}} \\
		\end{array}
	\right],
	\label{eq:out4_0}
\end{equation}
where $\hat{\bm{D}}$ is independent of the input matrix $\bm{B}$, as shown in the following equations: 
\begin{equation*}
	\bm{\hat{D}} = \bm{D}
	\left[
		\begin{array}{c}
			\frac{\bm{1}_4^\top}{m} \\[5pt]
			\bm{B}_{\mathrm{\Phi}}  \\
		\end{array}
	\right]^{-1} = 
	\left[
			\begin{array}{cccc}
				0 & 0 & g c_{\mathrm{x}4} & 0 \\
				0 & -g c_{\mathrm{y}4} & 0 & 0 \\
				c_{\mathrm{z}2} & 0 & 0 & 0 \\
				0 & 0 & 0 & c_{\mathrm{\psi}2} \\
			\end{array}
	\right].
\end{equation*}
In addition, $[\bm{1}_{4}/m, \bm{B}_{\mathrm{\Phi}}^T]^T$ is a non-singular matrix because of the non-zero $\bm{r}_i$.
From eq.\eqref{eq:out4_0}, the VG-ASSC uses the $i$-th element of the approximated output described by
\begin{equation}
	\begin{split}
	\hat{\bm{\eta}} &= \bm{G} \bm{\zeta} \\
	 &= \bm{G} \bm{C} \bm{\xi} + \bm{G} \bm{\hat{D}}
	\left[
		\begin{array}{c}
			\ddot{z}_{\mathrm{e}} \\
			\ddot{\bm{\phi}}_{\mathrm{e}} \\
		\end{array}
	\right] .
	\end{split}
	\label{eq:out4}
\end{equation}
%

\section{Numerical simulation}\label{sec_5}
As detailed in this section, the proposed and PID controllers are compared using numerical simulation.
We simulated the transportation of a rectangular-shaped payload using eight single-rotor robots and that of an L-shaped using ten robots.
In addition, the single-rotor robots use three types of robots with different maximum thrusts.
The proposed controller is designed to be robust against fluctuations in mass and COM and robot failures by utilizing the method described in Section \ref{sec_4-1-5}.

\begin{figure}[t]
	\centering
	\subfloat[Rectangle-shape]{
		\includegraphics[width=7cm]{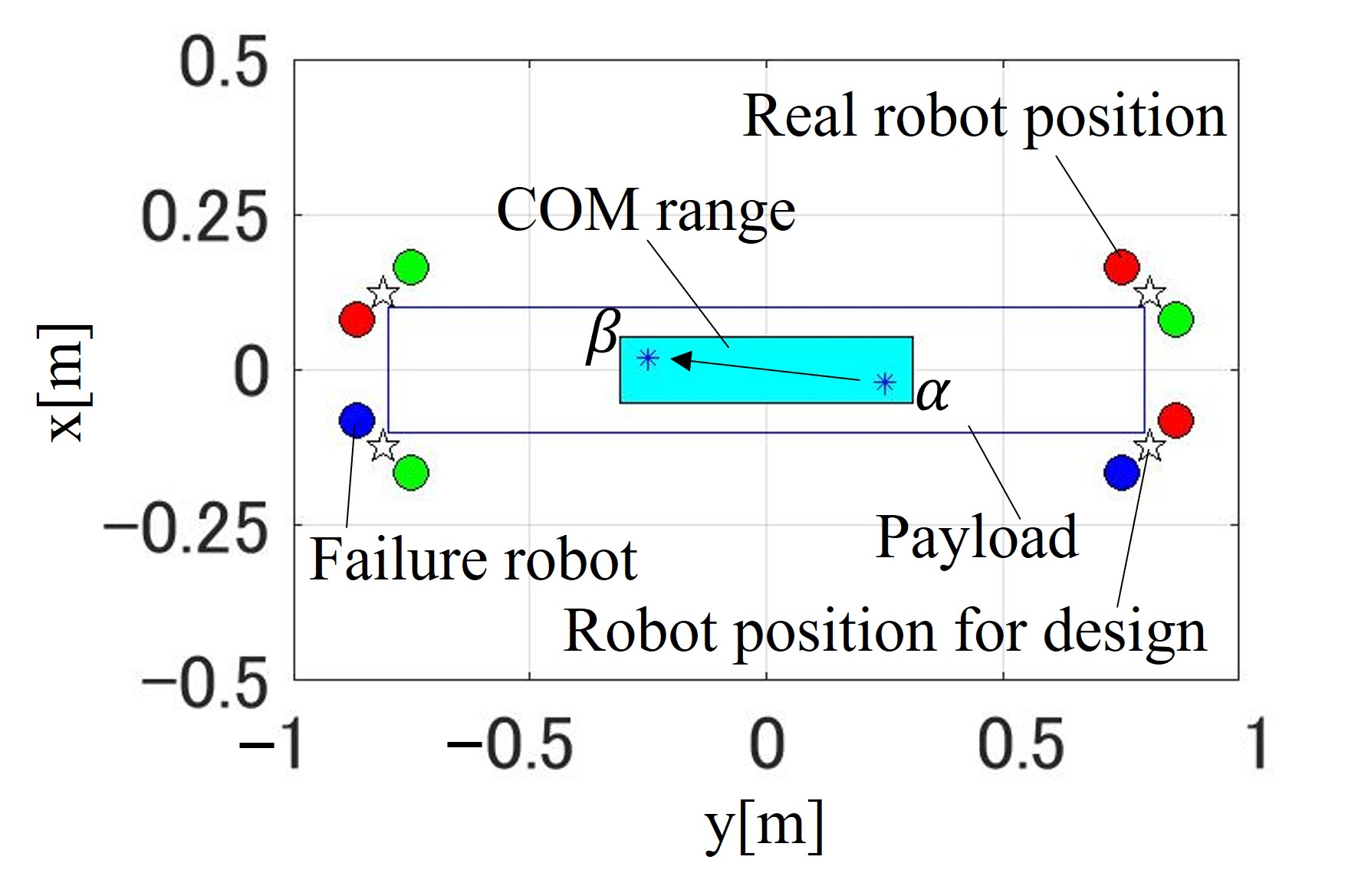}
		\label{fig:sim_c1}
	} 
	\subfloat[L-shape]{
		\includegraphics[width=7cm]{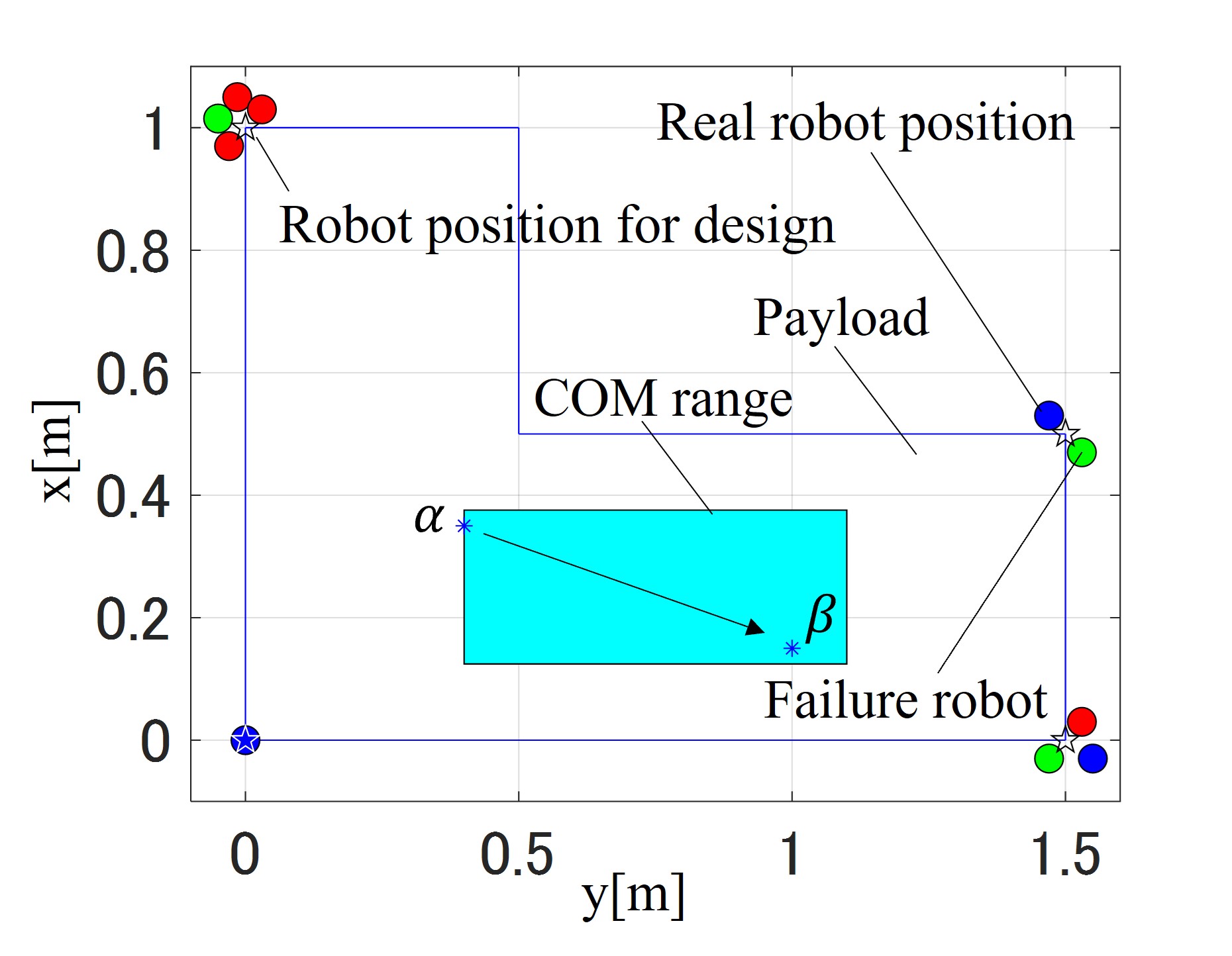}
		\label{fig:sim_c2}
	}
	\caption{Two different shapes of payloads used for the simulations. 
	The blue line indicates the shape of the payload, 
	filled circles indicate the actual robot positions, 
	red filled circles indicate Robot A,
	green filled circles indicate Robot B,
	blue filled circles indicate Robot C,
	stars indicate robot positions for the control design, 
	filled area indicates the COM range, 
	and asterisks indicate the simulated COM positions.}
	\label{fig:sim_c}
\end{figure}

\begin{table*}[t]
	\caption{Control parameters of simulation}
	\label{tb:sim_prm}
	\centering
	\begin{tabular}{c|cc}
			\hline
			Parameter & Rectangle-shape & L-shape\\ \hline \hline
			$n_1$, $n_2$, $n_3$, $n_4$  & $2$, $2$, $2$, $2$ & $4$, $2$, $3$, $1$ \\
			$m_1$,$m_2$ & $2.0$, $5.5$ kg & $2.0$, $5.5$ kg \\ 
			$c_{\mathrm{q}}$ & $0.162$ & $0.162$  \\
			$d_i$ & $ [-1,1-1,1]$ & $[-1,1,-1,1]$ \\
			$\bm{J}$ & $\rm{diag}\left(\left[0.419, 0.010, 0.429\right] \right)$ $\rm{kgm}^2$  & $\rm{diag}\left(\left[0.392, 0.142, 0.521\right]\right)$ $\rm{kgm}^2$  \\
			$[c_{\mathrm{x}0}, c_{\mathrm{x}1}, c_{\mathrm{x}2}, c_{\mathrm{x}3}, c_{\mathrm{x}4}]$ & $[1.00, 2.20, 1.82, 0.67, 0.09]$  & $[1.00, 2.00, 1.5, 0.5, 0.06]$ \\
			$[c_{\mathrm{y}0}, c_{\mathrm{y}1}, c_{\mathrm{y}2}, c_{\mathrm{y}3}, c_{\mathrm{y}4}]$ & $[1.00, 2.00, 1.5, 0.5, 0.06]$  & $[1.00, 2.00, 1.5, 0.5, 0.06]$ \\
			$[c_{\mathrm{z}0}, c_{\mathrm{z}1}, c_{\mathrm{z}2}]$ & $[1.00, 1.00, 0.25]$  & $[1.00, 1.00, 0.25]$ \\
			$[c_{\mathrm{\psi} 0}, c_{\mathrm{\psi} 1}, c_{\mathrm{\psi} 2}]$ & $[1.00, 1.00, 0.25]$ & $[1.00, 1.00, 0.25]$ \\
			$[\tau_1, \tau_2, \tau_3]$ & $[0.25, 140, 1.00]$ &  $[0.40, 60.0, 0.60]$ \\
			$\overline{k}_i^j$ & $49$ & $49$\\
			$\underline{k}_i^j$ & $7$ & $7$ \\
			\hline
	\end{tabular}
\end{table*}

\begin{figure}[t]
	\centering
		\includegraphics[keepaspectratio, scale=0.35]{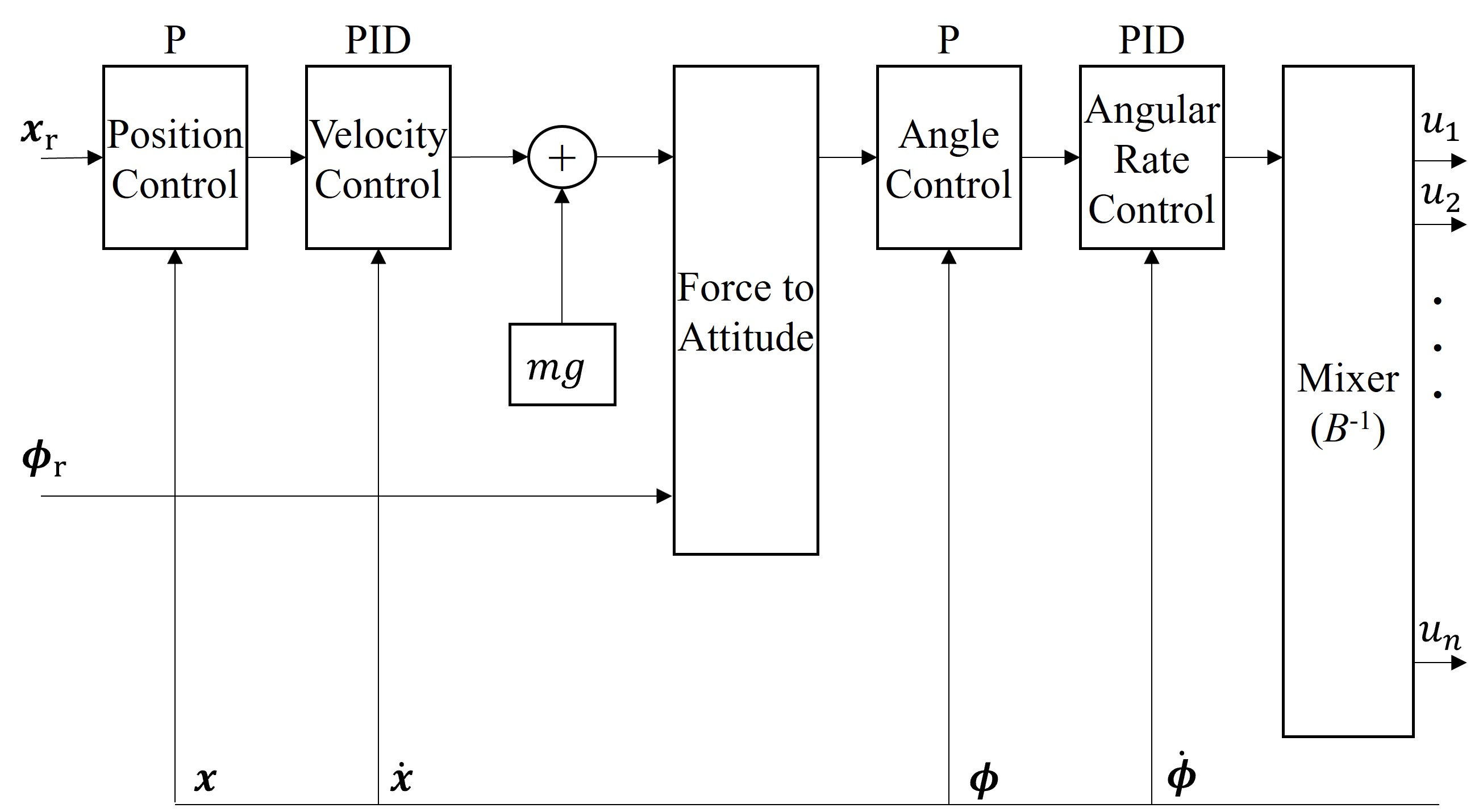}
	  \caption{Conventional PID controller}\label{fig:pid_arch}
\end{figure}

\begin{table}[t]
	\caption{Simulation conditions}
	\label{tb:sim_env}
	\centering
	\begin{tabular}{c|c}
			\hline
			Condition  & Value \\
			\hline \hline
			Initial position & $\left[ 0.0,0.0,2.0 \right] $ m \\ 
			Initial attitude  & $\left[ 0.0,0.0,0.0 \right] $ rad \\ 
			Target position & $\left[ 3.0,0.0,2.0 \right] $ m \\  
			Target attitude & $\left[ 0.0,0.0,0.0 \right] $ rad \\ 
			Time to give command & $30$ s \\ 
			Time to give COM fluctuations ($t_{\mathrm{c}}$) & $50$ s \\ 
			Time to give failure ($t_{\mathrm{f}}$) & $60$ s \\ 
			Initial mass & $3.0$ kg \\ 
			Mass after fluctuations & $4.0$ kg \\
			Motor time constant & $0.01$ s \\
			Time step & $10$ {\textmu}s \\
			\hline
	\end{tabular}
\end{table}
\subsection{Condition}\label{sec_5-1}
In this simulation, we used three robots, designated A, B, and C, with different maximum thrusts of $7$, $12$, and $15$ N, respectively.
A configuration of simulation targets is shown in Fig.\ref{fig:sim_c}.
The rectangular shape is similar to the prototype described in Section \ref{sec_6}.
The arrows in the figure indicate the movement of the COM during the simulation.
The parameters used for the controller design are presented in Table \ref{tb:sim_prm}.
Moreover, our controller requires an equilibrium thrust of $u_{0i}^j$. However, $u_{0i}^j$ is unknown at the time of the design.
Therefore, the initial values are given by equally distributing $3.5$ kg, which is the average of the mass design range, among the robots.
Thus, the robots use fixed gain instead of variable gain until they obtain the true value of $u_{0i}^j$.
The fixed gain at this time is $\underline{k}_i^j$.
In addition, the $\overline{k}_i^j$ and $\underline{k}_i^j$ of all robots are the same.
The PID control for comparison was the general cascade PID shown in Fig. \ref{fig:pid_arch} \cite{morin2019}.
The parameters of the PID controller were set such that the difference in response time from that of the proposed control would be $10$ \%
when the reference value given a step under the COM was the center and the mass was $3.5$ kg.

The simulation task involves moving from an initial position in mid-air to a target position given as a reference.
The proposed controller requires the thrust $u_{0i}^j$ of each robot during hovering as a control parameter.
Therefore, the thrust at 18 s is obtained as $u_{0i}^j$.
Subsequently, the target position is set at 30 s, and fluctuations in the COM and mass are set at 50 s ($t_{\mathrm{c}}$).
In addition, the robot failure is set at 60 s ($t_{\mathrm{f}}$).
Simulation parameters are presented in Table \ref{tb:sim_env}.
Furthermore, if the pitch or roll angle exceeds 90$^\circ$ or the altitude is less than zero, it is considered to have crashed.

The simulation is performed using a model in which the Coriolis force and a first-order approximation of the motor are included in eq. (\ref{eq:dy}).
Additionally, the time step is 10 {\textmu}s, which is a significantly shorter period than the system response time.

\begin{figure}[t]
	\centering
	\subfloat[PID controller]{
		\includegraphics[width=6.5cm]{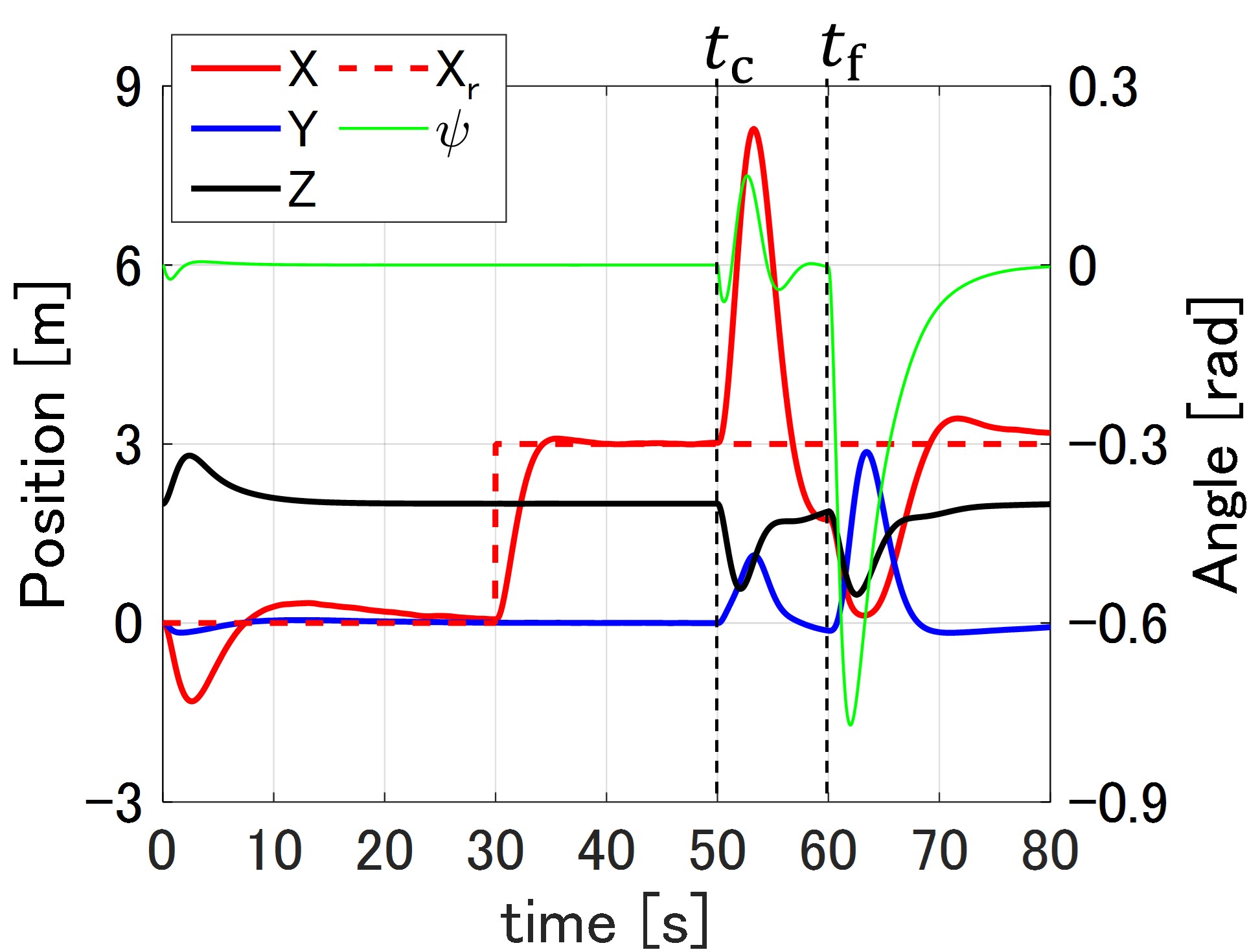}
		\label{fig:sim_rec2}
	} \hspace{5mm}
	\subfloat[Proposed controller]{
		\includegraphics[width=6.5cm]{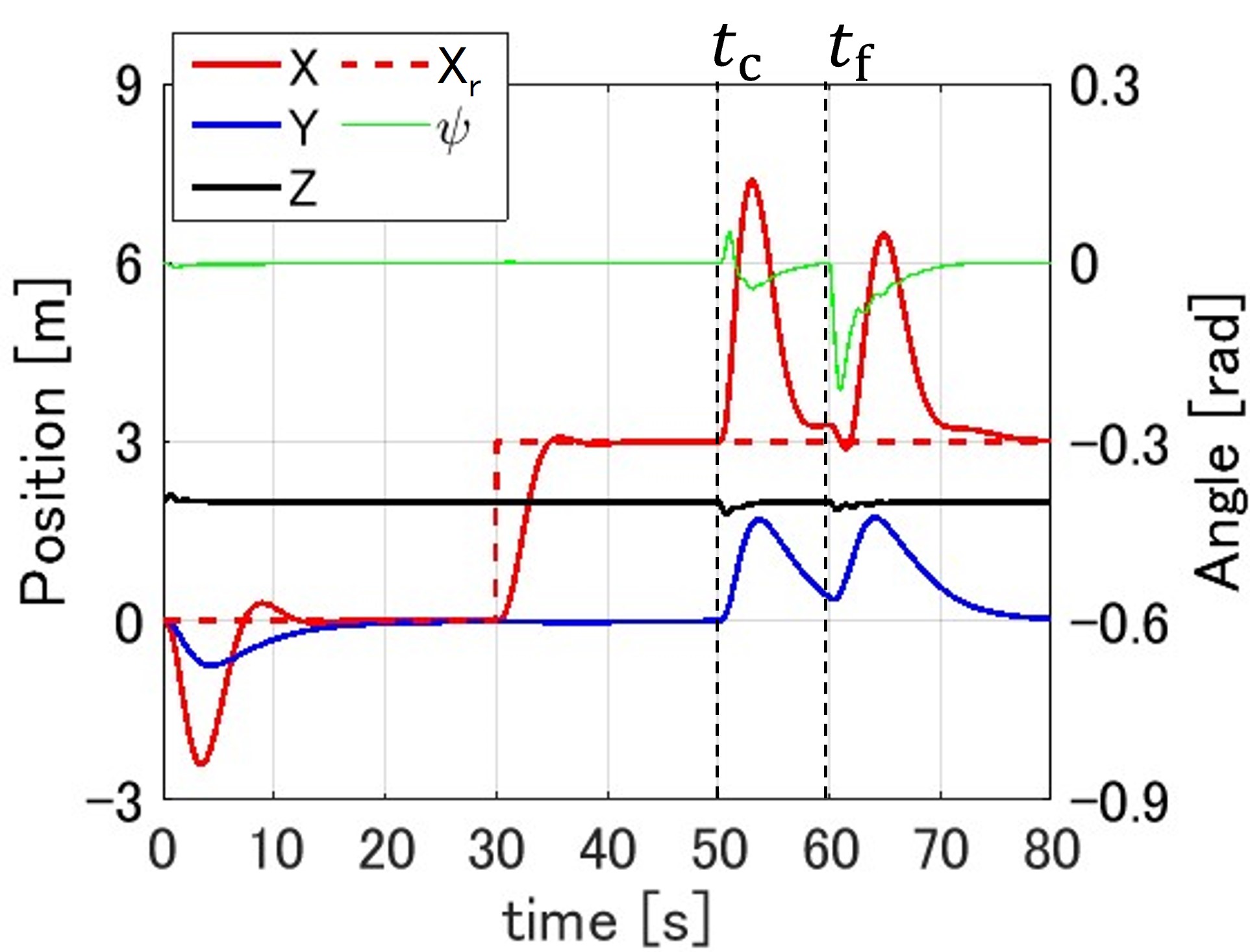}
		\label{fig:sim_rec1}
	}
	\caption{Time series of the position, yaw angle, and reference of the
	rectangular-shaped payload with fluctuations in mass and COM and robot failure using two controllers.
	$X_r$ denotes the reference of $X$.
	The vertical axes on the left and right indicate the positions and yaw angle, respectively.
	}
	\label{fig:sim_rec}
\end{figure}

\subsection{Result}\label{sec_5-2}
The simulation results with the rectangular- and L-shaped payloads are illustrated in Fig. \ref{fig:sim_rec} and Fig. \ref{fig:sim_L}, respectively.
$X$, $Y$, and $Z$ are the three-dimensional positions on the world frame transformed using the yaw angle $\psi$.
Compared with that of the PID controller, the overshoot of the proposed controller is almost identical but the disturbance-induced fluctuation is less pronounced for both the rectangular- and L-shaped payloads. 
In particular, the PID controller crashes at 52 s because of the altitude condition, whereas the proposed controller continues to fly in the case of the L-shaped payload.
Therefore, the proposed controller is expected to be adaptable to changes in geometry and the number of robots and can be robust against fluctuations.

\begin{figure}[t]
	\centering
	\subfloat[PID controller]{
		\includegraphics[width=6.5cm]{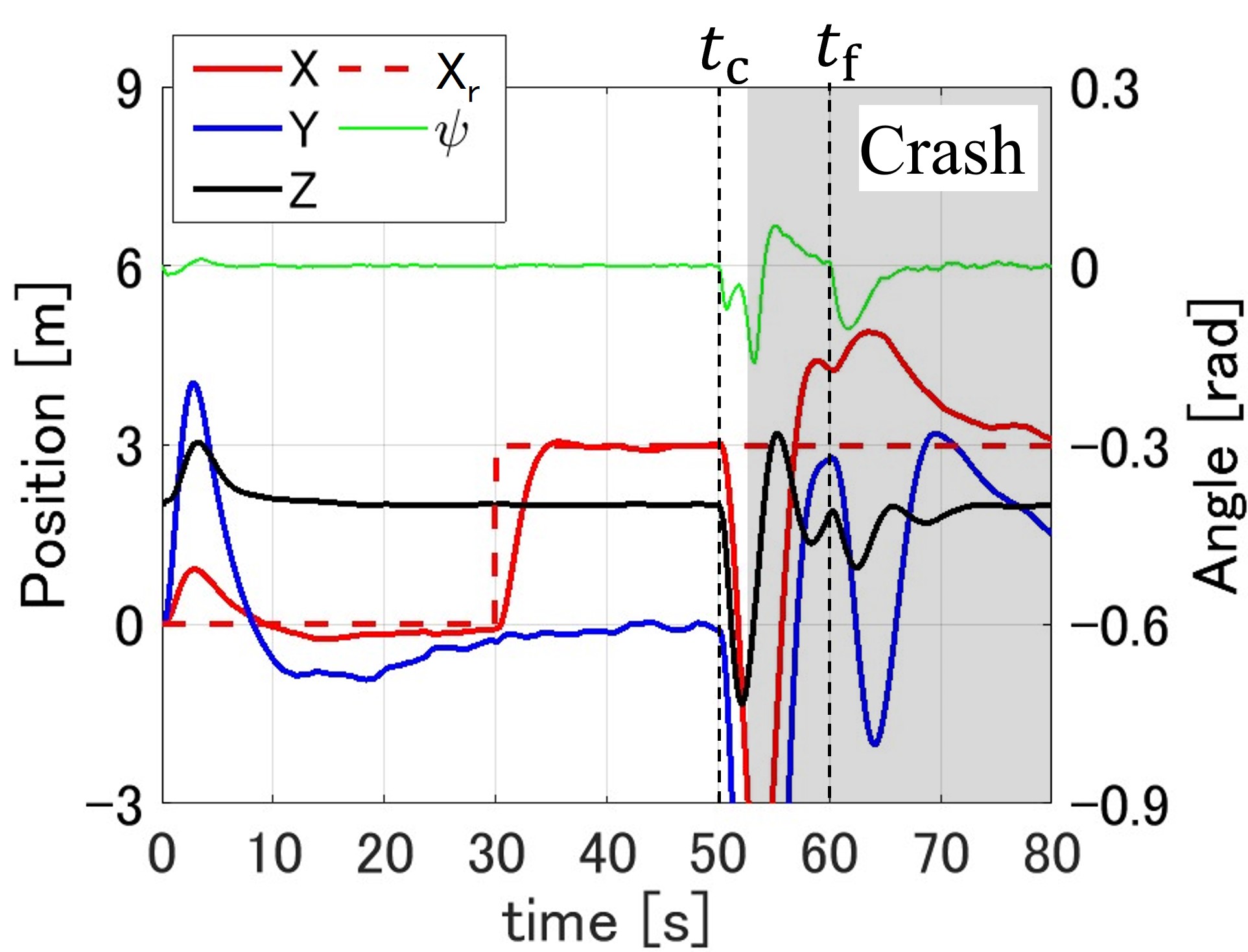}
		\label{fig:sim_L2}
  } \hspace{5mm}
	\subfloat[Proposed controller]{
		\includegraphics[width=6.5cm]{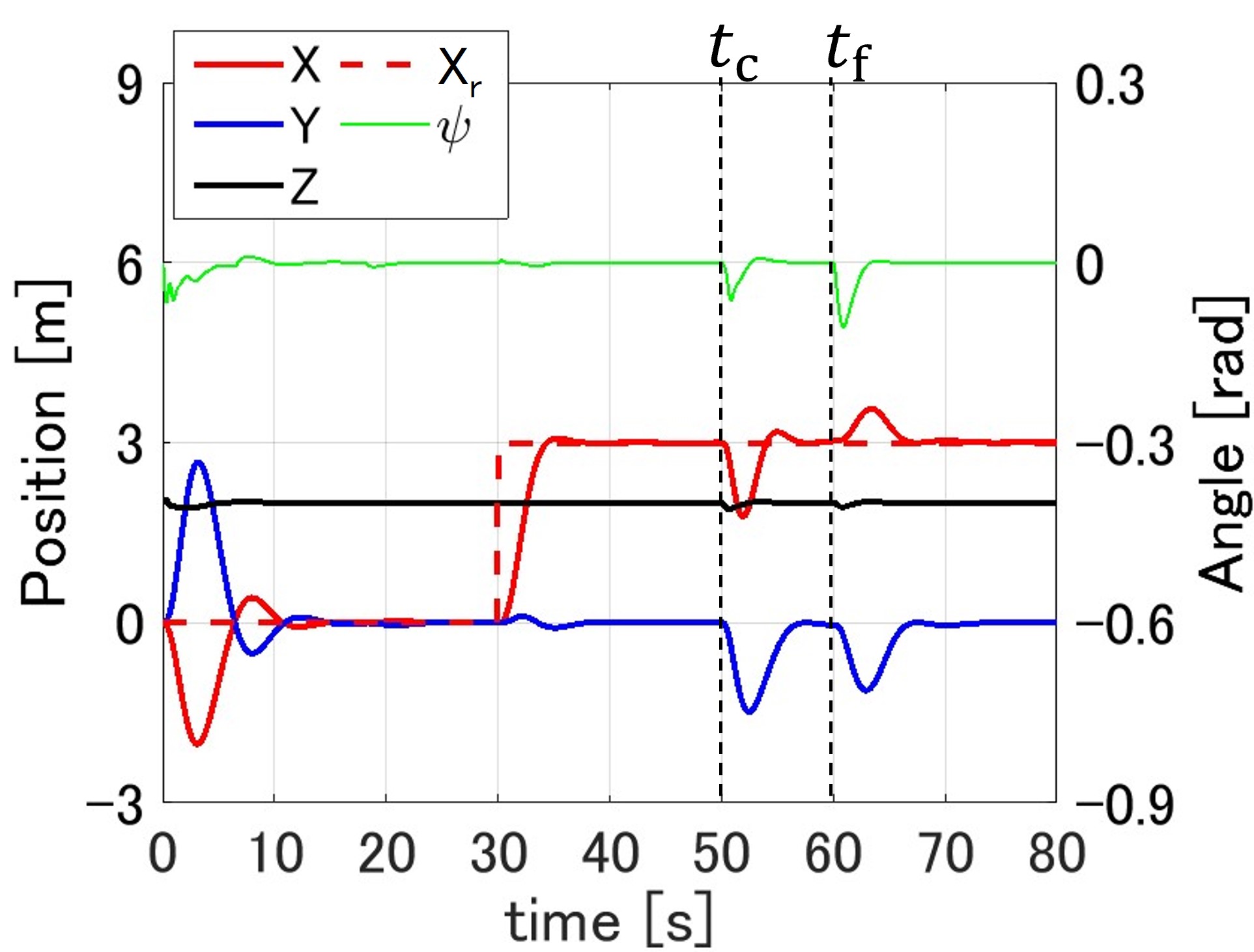}
		\label{fig:sim_L1}
	}
	\caption{Time series of the position, yaw angle, and reference of the
	L-shaped payload with fluctuations in mass and COM and robot failure using two controllers. 
	Each display is the same as that shown in Fig.\ref{fig:sim_rec}.}
	\label{fig:sim_L}
\end{figure}

\section{Real robot experiment}\label{sec_6}
This section presents the implementation of the proposed method on a prototype consisting of various types of single-rotor robots and verification of the feasibility of aerial transportation.
The prototype has a rectangular-shaped payload and eight single-rotor robots.
The experiment moves from takeoff to destination with fluctuations.

\subsection{Prototype}\label{sec_6-1}
The prototype comprises eight robots, a rectangular-shaped payload, and ball robots that simulate changes in COM, as shown in Fig.\ref{fig:proto}.
The eight robots comprise three types of robots: robot $\rm{A}'$, robot $\rm{B}'$, and robot $\rm{C}'$.
Each robot is distributed and controlled by an independent controller, as discussed in Section \ref{sec_4-2-2}.
However, the controller of the prototype is independent only within the software to facilitate production.
Therefore, regarding the hardware, each robot is controlled by a unique controller mounted on the payload.
The controller hardware is PIXHAWK \cite{pix2011}.
The software for the controller was coded using MATLAB\textsuperscript{\tiny\textregistered} Simulink\textsuperscript{\tiny\textregistered},
and Stateflow\textsuperscript{\tiny\textregistered} and was implemented using PX4 Autopilots Support from UAV Toolbox.
Each position, velocity, and yaw angle were acquired using a motion capture system, 
and other states were acquired using the existing estimator from the sensors in the controller \cite{pix2011}.
The control period was set to $5$ ms.
The prototype provided space for the ball robots to move for COM fluctuations on the top of the payload.
SPRK+ was used for the ball robots \cite{sphero}.
The range of COM change by the ball robots is $0.12$ m on the left and $0.03$ m on the right, as shown in Fig. \ref{fig:proto}.
The offset to the left is the effect of the mounted camera for photographing robot failure.
In this experiment, the ball robots were remotely controlled at a certain time to imitate fluctuation in COM.
In addition, failure was imitated by forcibly setting the thrust command of the single-rotor robot to $0$.
The specifications of the prototype are presented in Table \ref{tb:spec}.

\begin{figure}[t]
	\centering
		\includegraphics[keepaspectratio, scale=0.3]{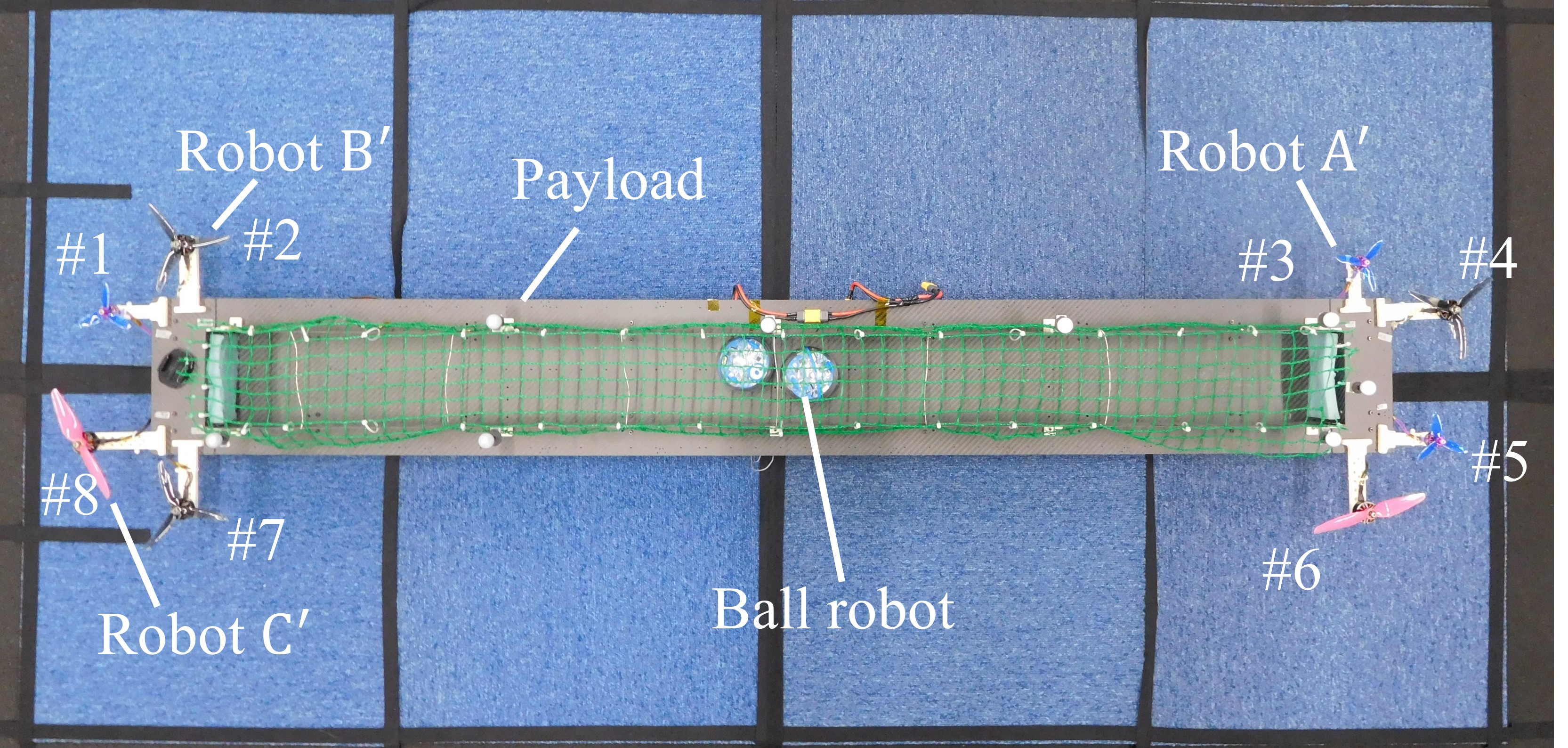}
	  \caption{Prototype}\label{fig:proto}
\end{figure}

\begin{table}[t]
	\caption{Prototype specifications}
	\label{tb:spec}
	\centering
	 \begin{tabular}{|cc|}
		\hline
		\multicolumn{2}{|c|}{Payload} \\
		\hline
	   Mass (with all robots)  & 2.2 kg \\
		 Size (without robots)  & $1.6$ $\times$ $0.2$ m \\
		 Mass of battery & $0.43$ kg \\
		 Inertia($[\bm{J_{xx}}, \bm{J_{yy}}, \bm{J_{zz}}]$) &  $\left[0.419, 0.010, 0.429\right]$ $\rm{kgm}^2$ \\
		 Mass of Ball robot & $0.18$ kg \\
		 Mass of camera & $0.18$ kg  \\
		 \hline \hline
		 \multicolumn{2}{|c|}{Robot $\rm{A}'$} \\
		 \hline 
		 Motor KV & $4200$ KV \\
		 Propeller & $3$ inch, $3$ blades \\
		 Max-thrust (catalog value) & $6.0$ N ($0.61$ kgf) \\
		 \hline \hline
		 \multicolumn{2}{|c|}{Robot $\rm{B}'$} \\
		 \hline
		 Motor KV & $2400$ KV \\
		 Propeller & $5$ inch, $3$ blades \\
		 Max-thrust (catalog value) & $14.2$ N ($1.45$ kgf) \\
		 \hline \hline
		 \multicolumn{2}{|c|}{Robot $\rm{C}'$} \\
		 \hline
		 Motor KV & $2500$ KV \\
		 Propeller & $6$ inch, $2$ blades \\
		 Max-thrust (catalog value) & $16.7$ N ($1.7$ kgf) \\
		\hline

	 \end{tabular}
\end{table}

\begin{figure}[t]
	\centering
		\includegraphics[keepaspectratio, width=7cm]{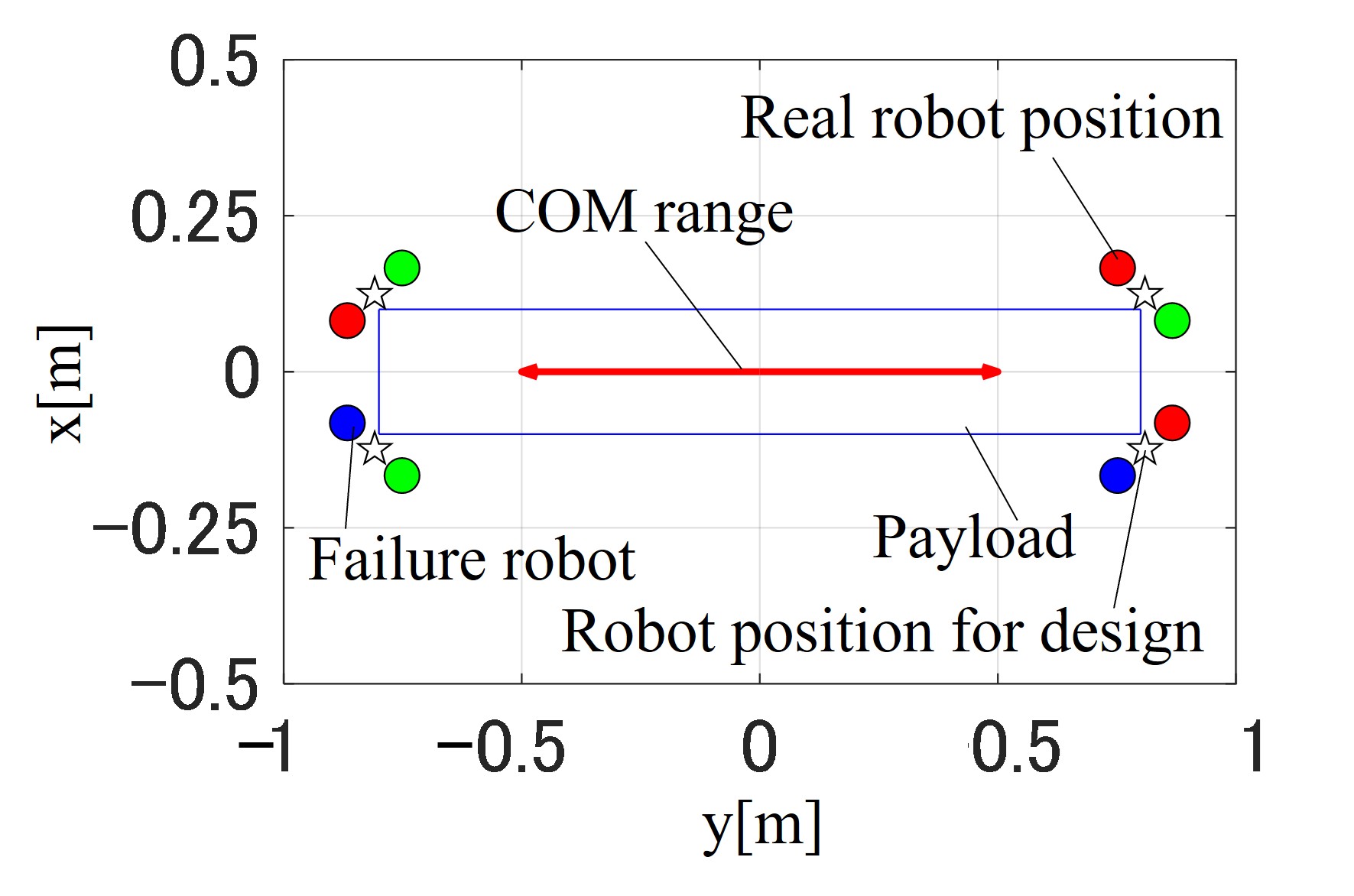}
	  \caption{Robot position and COM range on control design and failure robot.
		Filled circles indicate actual robot positions, 
		red filled circles indicate Robot A',
		green filled circles indicate Robot B',
		blue filled circles indicate Robot C', and
		stars indicate robot positions for the control design.}\label{fig:real_con}
\end{figure}

\begin{table}[t]
	\caption{Control parameters of real robot experiment}
	\label{tb:demo_c_prm}
	\centering
	\begin{tabular}{c|c}
		  \hline
			 Parameter & Prototype \\
			 \hline \hline
			$n$  & $8$ \\ 
			$m_1,m_2$ & $3.1$, $3.1$ kg \\  
			$c_{\mathrm{q}}$ & $0.162$  \\ 
			$d_i$ & $[-1,1-1,1]$ \\
			$\bm{J}$ & $ \rm{diag}\left(\left[0.419, 0.010, 0.429\right]\right)$ $\rm{kgm}^2$  \\ 
			$\left[c_{\mathrm{x}0}, c_{\mathrm{x}1}, c_{\mathrm{x}2}, c_{\mathrm{x}3}, c_{\mathrm{x}4}\right]$ & $\left[1, 4, 3, 1, 0.01 \right]$ \\ 
			$\left[c_{\mathrm{y}0}, c_{\mathrm{y}1}, c_{\mathrm{y}2}, c_{\mathrm{y}3}, c_{\mathrm{y}4}\right]$ & $\left[1, 3, 3, 0.2, 0.01 \right]$ \\ 
			$\left[c_{\mathrm{z}0}, c_{\mathrm{z}1}, c_{\mathrm{z}2}\right]$ & $\left[1, 0.8, 0.05\right]$ \\ 
			$\left[c_{\mathrm{\psi} 0}, c_{\mathrm{\psi} 1}, c_{\mathrm{\psi} 2}\right]$ & $\left[1, 7, 0.25\right]$ \\ 
			$\left[\tau_1, \tau_2, \tau_3\right]$ & $\left[0.1, 18, 2\right]$ \\ 
			$\overline{k}_i^j$ & $4.5$ \\ 
			$\underline{k}_i^j$ & $1.5$ \\
			\hline
	\end{tabular}
\end{table}

\begin{figure*}[t]
	\centering
	\subfloat[Trajectory]{
		\includegraphics[width=6cm]{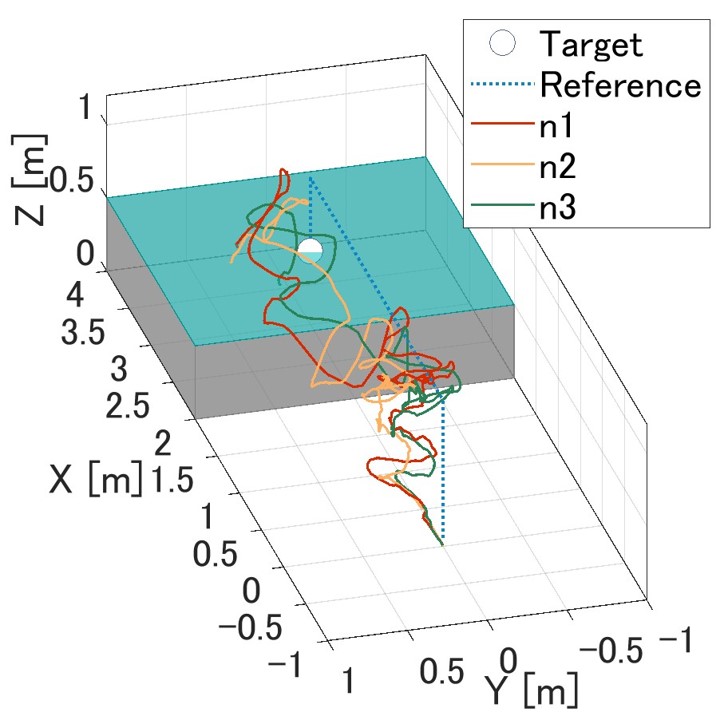}
		\label{fig:real_3d}
	} 
	\subfloat[Time series of position and yaw angle]{
		\includegraphics[width=9cm]{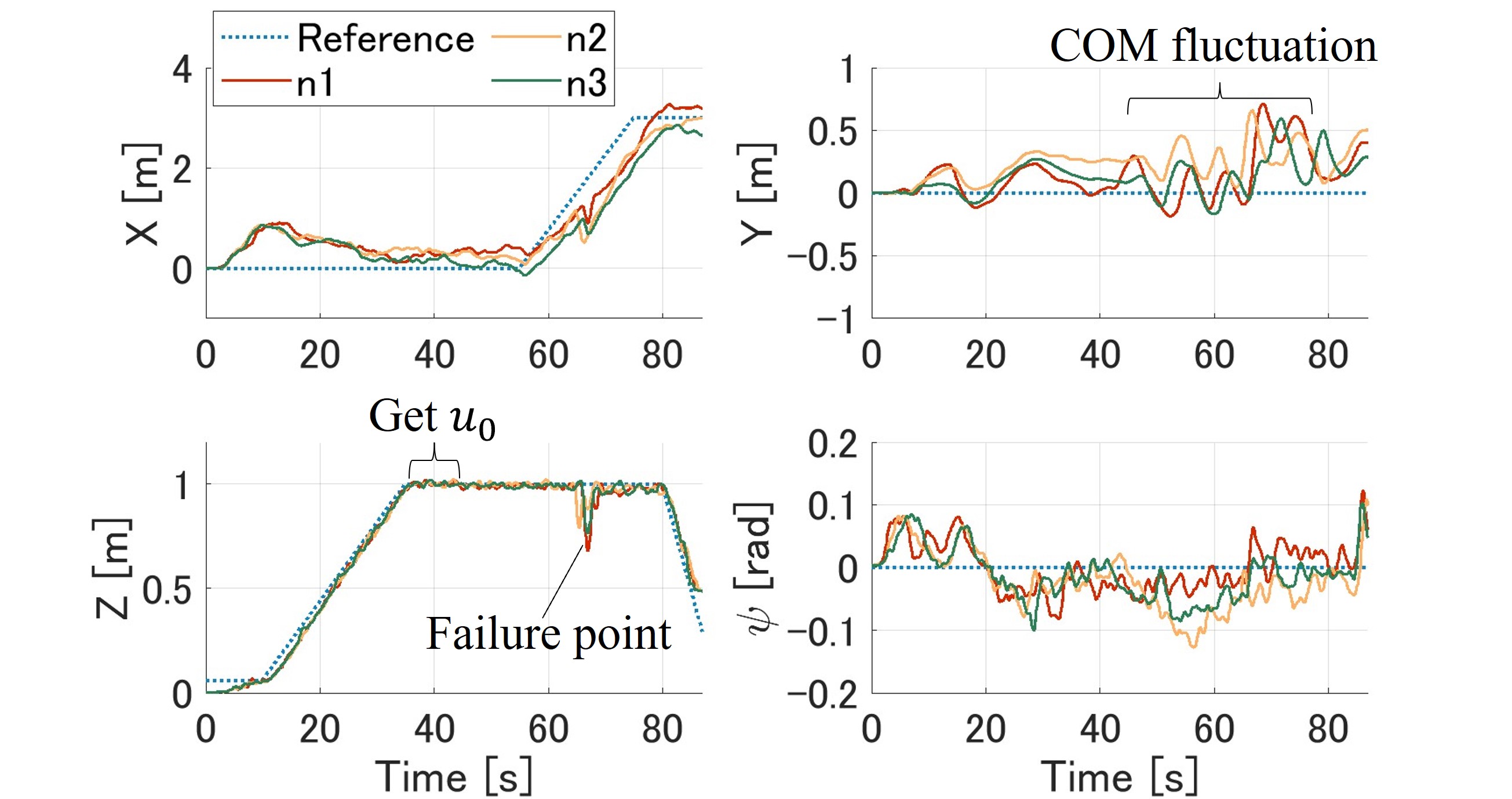}
		\label{fig:real_t}
	}
	\caption{Payload positions and yaw angle during experiments with fluctuations. 
	Dotted lines represent the reference trajectory, and solid lines represent the actual positions.}
	\label{fig:real}
\end{figure*}

\begin{figure*}[t]
	\centering
		\includegraphics[keepaspectratio, width=14.5cm]{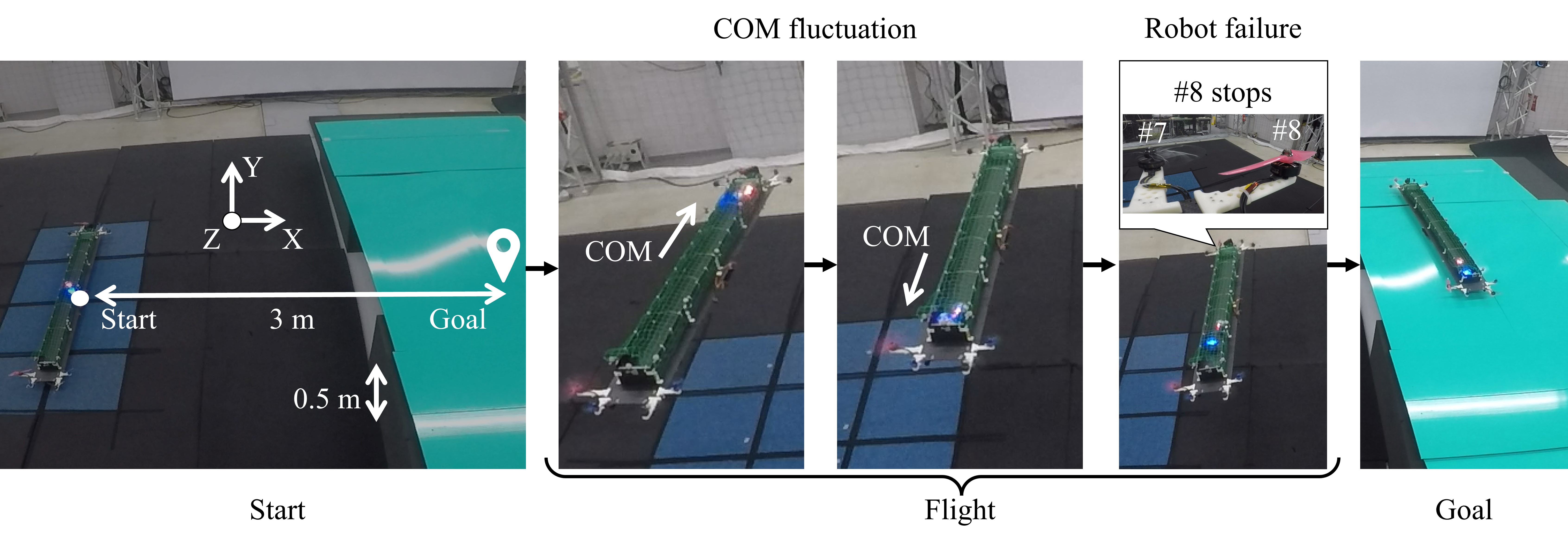}
	  \caption{Experiment using prototype}\label{fig:exp}
\end{figure*}

\begin{figure*}[t]
	\centering
	\subfloat[Trajectory]{
		\includegraphics[width=6cm]{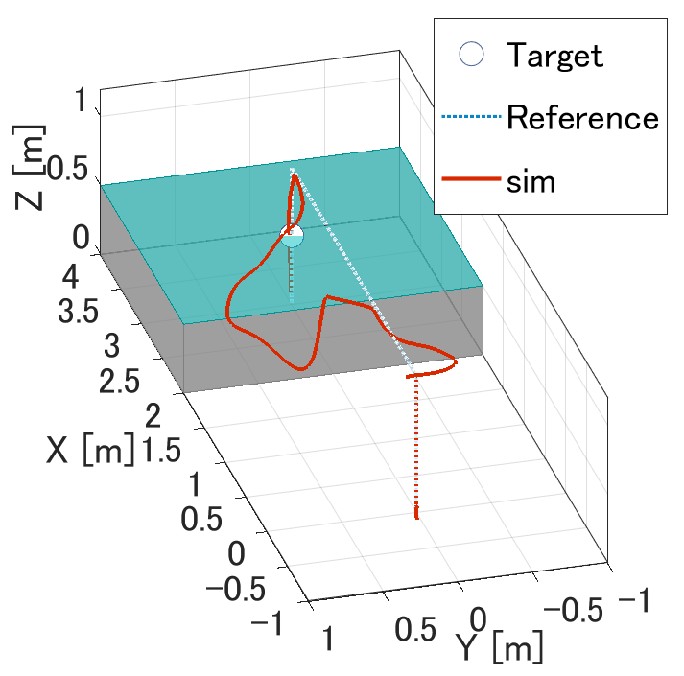}
		\label{fig:sim_3d}
	} 
	\subfloat[Time series of position and yaw angle]{
		\includegraphics[width=9cm]{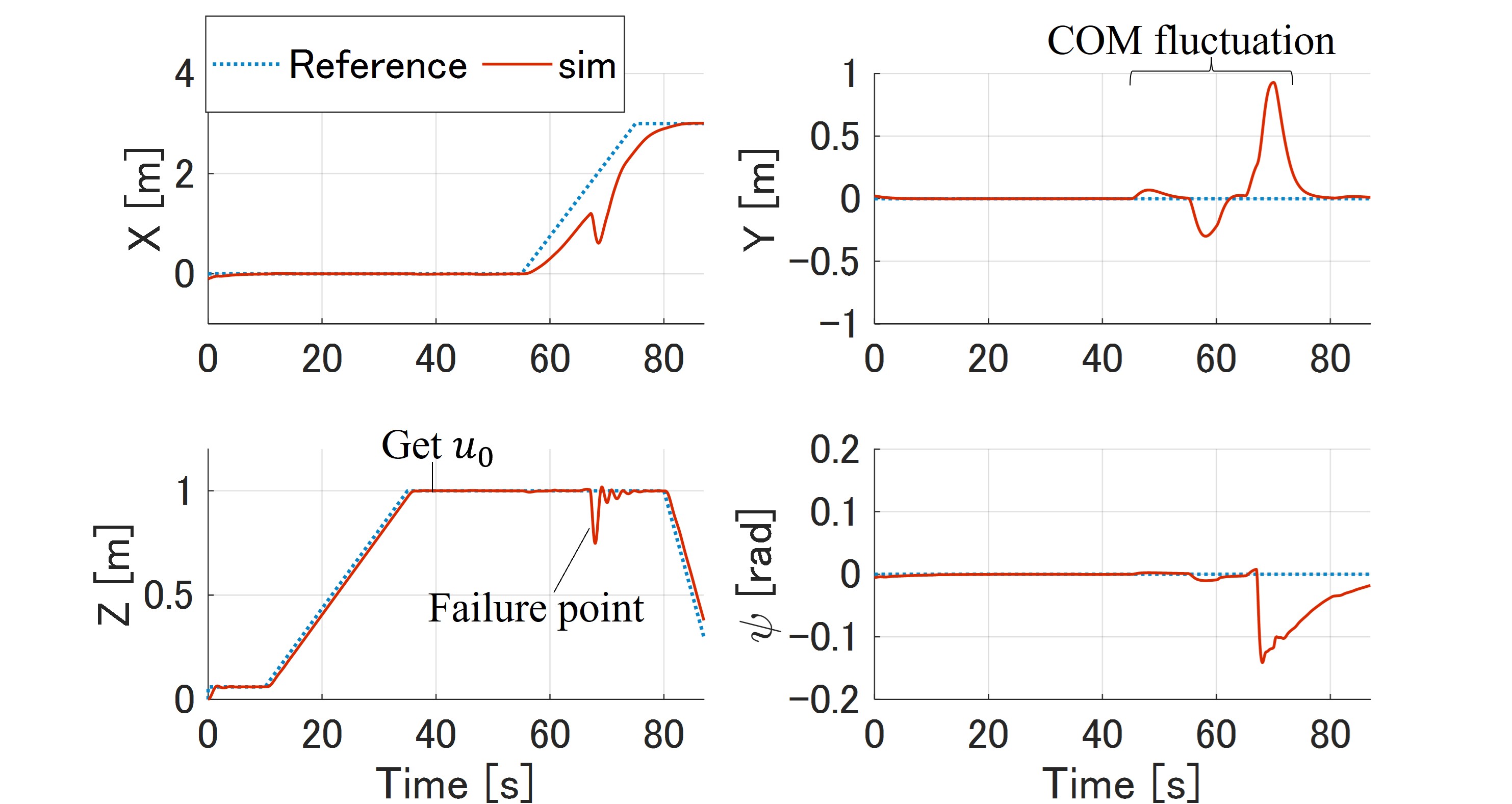}
		\label{fig:sim_t}
	}
	\caption{Payload positions and yaw angle during simulation under the same conditions as the experiments. }
	\label{fig:sim_under_real}
\end{figure*}

\subsection{Condition}\label{sec_6-2}
The real experimental task comprised the transportation to the target positions and disturbances.
The controller was designed to be robust against fluctuations in COM and robot failure because mass fluctuation can not easily be imitated in real experiments.
The fluctuation range of the COM for design was limited to the longitudinal direction because the COM of the prototype fluctuated in the longitudinal direction, as shown in Fig. \ref{fig:real_con}.
The representative point of robot positions for design was set to the midpoint of the line connecting robots in each quadrant.
The parameters in the control design are presented in Table \ref{tb:demo_c_prm}.
The initial values of $u_{0i}^j$ was set by dividing a total mass $3.1$ kg equally among the $n$ robots.
The true value of $u_{0i}^j$ was obtained during hovering.
Additionally, $X$, $Y$, and $Z$ are the three-dimensional positions on the world frame as in the simulation.

The transportation tasks in this experiment included takeoff, hovering in the air, and landing at the target position.
Specifically, the payload was lifted at an altitude of 1 m, then moved $3.0$ m in the $X$-direction, and landed on a $0.5$-m-high platform.
The reference value $\bm{x}_\mathrm{r}$ for performing this operation was continuously and automatically given.
Furthermore, the controller used $\underline{k}_i^j$ as a fixed gain, rather than a variable gain, until the true value of $u_{0i}^j$ was obtained during hovering.
To obtain $u_{0i}^j$, the thrust command of each robot with a 5.0-Hz low-pass filter was used.
In this experiment, for safety, the acquisition of $u_{0i}^j$ and robot failure was performed when a signal was remotely transmitted.
In addition, the movement of the ball robots following a change in COM was performed after obtaining $u_{0i}^j$.

\subsection{Result}\label{sec_6-3}
The results of the three experiments are shown in Fig. \ref{fig:real}.
In all three experiments, 
$u_{0i}^j$ was obtained at approximately 40 s, the COM was changed by the ball robots at 50 s when the command in the $X$-direction was given until it reached the goal, 
and robot failure occurred at approximately 67 s.
Fig. \ref{fig:exp} shows the flow of the third experiment.
Fig. \ref{fig:real_3d} and Fig \ref{fig:real_t} show that the proposed controller can perform aerial transportation even with heterogeneous robot configurations.
Moreover, the prototype could reach the vicinity of the goal without falling into a critical state when COM fluctuations and robot failure occurred.
Fig.~\ref{fig:sim_under_real} shows the simulation results under the same conditions as the experiment.
Regarding the control at height $Z$, the prototype exhibited a behavior similar to that observed in the simulation.
A discrepancy in the yaw angle $\psi$ between the simulation and the experiment was observed, probably due to the weak yaw torque of the multi-copter-type aerial vehicle, rendering it highly susceptible to external disturbances \cite{yaw}.
For the other values, although some fluctuations were observed, they were generally close to those in the simulation.

\section{Discussion}\label{sec_7}
Cooperative transportation by the proposed decentralized controller is possible even with heterogeneous robot configurations under the same constraints as those used in previous studies \cite{Oishi2021ICRA}.
In addition, the simulations discussed in Section \ref{sec_5} indicate that numerous robots, which were not mentioned in previous studies, can be controlled.
The constraints used in our approach are more alleviated than those in the approaches of \cite{KumarRigid2013} and \cite{pereira2018} in terms of the mass and shape of the payload.
The constraints on the robot position are assumed to be concentrated at representative points in each quadrant, as shown in Fig. \ref{fig:pory}.
These constraints are similar to the symmetry arrangement in the approach of \cite{wang2018}. 
These positional constraints are set because the controller requires specific positional information to guarantee stability.
An approach wherein the robot estimates its position with respect to the payload to alleviate restrictions on position information can be applied  \cite{Oishi2021IROS}.
The proposed method has a limitation in that it assumes the existence of a solution for the LMI of the RFC.
Depending on the uncertainty range in Section~\ref{sec_4-1-4}, no solution may exist.
Although deriving the conditions for the existence of the solution analytically is challenging, the feasibility conditions can be determined via numerical approaches \cite{LMIcond}.

For hardware, we use a single-rotor robot such as that employed in \cite{oung2011,mu2019universal}.
This study shows the possibility of using a modeled aerial robot similar to that employed in \cite{mu2019universal} to perform transportation tasks with distributed controllers and heterogeneous robots.
Thus, our research result further improves the advantages of modular configurations.
An aerial robot equipped with a single rotor rather than a multi-copter has a wider range of applications in assuming cooperative flight with multiple robots because a multi-copter is a collection of rotors equipped with a single rotor.
Connecting with a payload can be challenging for transportation.
The connection between single-rotor robots and a payload is screwed in the proposed prototype.
This connection should be simplified to render plug-in/plug-out practical.
Notably, methods applied as grasping mechanisms for multi-copters have been proposed in recent years \cite{meng2022aerial}.

During the experiment,
the proposed controller could reach the goal even when fluctuations occurred. However, an error was detected immediately after the start and near the goal for the position, as shown in Fig. \ref{fig:real}.
This error may be attributed to the ground effect, which occurs in the vicinity of the ground.
Fishman et al. \cite{fishman2021dynamic} attempted to address this ground effect using a deep learning method.
However, they did not succeed because the data obtained were influenced by issues such as contact near the ground.

\section{Conclusion}\label{sec_8}
In this study, we proposed a decentralized cooperative transportation system with heterogeneous single-rotor robots.
The proposed method extended plug-in/plug-out, the advantages of decentralized control, to heterogeneous robots for cooperative aerial transportation tasks.
Thus, even deteriorated robots could be reused.
In addition, the proposed controller using an RFC and VG-ASSC was robust against fluctuations while maintaining the performance of conventional controllers, and convergence was guaranteed.
Numerical simulation showed that the proposed system could be transported in the air under significant fluctuations even with different numbers of robots and different payload shapes.
In addition, 
aerial transportation was possible even if a failure and COM shift occurred in real robot experiments using the prototype.

As a future task, the restrictions of this controller must be alleviated.
The robot positions are restricted to the vicinity of the representative point to render the target system SPR.
This restriction is a major limitation as rotor blades are generally large.
In the future, we will aim to alleviate this restriction.

\bibliographystyle{tfnlm}
\bibliography{dist_hetero}

\appendix
\section{Proof of theorem\ref{thm:assc_a}}\label{sec_proof1}
The error system of eq.(\ref{eq:st_fin}) can be represented as follows:
\begin{equation}
	\begin{array}{l}
		\bm{\dot{\xi}} = \hat{\bm{A}} \bm{\xi} + \bm{B} \hat{\bm{U}} \\
		\bm{\eta} = \bm{G} \hat{\bm{C}} \bm{\xi} + \bm{G} \bm{D} \hat{\bm{U}}
	\end{array},
	\label{eq:sys2}
\end{equation}
where $\hat{\bm{A}} := \bm{A} - \bm{B} \bm{F}$, $\hat{\bm{C}} := \bm{C} - \bm{D} \bm{F}$ and $\hat{\bm{U}} := \bm{U}_{\rm s} - \bm{U}_{\mathrm{r}}$.
The hyperstability theorem \cite{anderson1968,landau1984} ensures that the asymptotic stability of the error $\bm{\eta}$ if eq.(\ref{eq:assc}) is passive.
This is because eq.(\ref{eq:sys2}) is SPR for the RFC.
Therefore, we show that eq. (\ref{eq:assc}) is passive.

From eq.~\eqref{eq:assc}, the sum of outputs of controllers in quadrant $i$ is given as
\begin{equation}
		\bm{v}_i := -U_{\mathrm{s}i} + U_{\mathrm{r}i} = \sum_{j=1}^{n_i} v_i^j,
		\label{eq:out_v}
\end{equation}
where
\begin{equation*}
	v_i^j := - u_{{\mathrm s}i}^j + u_{{\mathrm r}i}^j .\\
\end{equation*}
The entire system can be represented as shown in Fig.~\ref{fig:system_p} using eqs.\eqref{eq:assc}, \eqref{eq:sys2}, and \eqref{eq:out_v}.
Although this system has multiple outputs, one quadrant is a single output.
Thus, the methodology used by \cite{Amano} can be employed to prove the passivity of eq.~\eqref{eq:assc} for each quadrant.
In addition, $\bm{v}_{i}$ can be expressed as
\begin{align}
	\begin{split}
		\hspace{-10mm} \bm{v}_{i} &= \sum_{j=1}^{n_i}\left( - u_{{\mathrm s}i}^j + u_{{\mathrm r}i}^j \right) = \sum_{j=1}^{n_i}\left( - \rho ({\varphi_{i}^j}) - u_{0i}^j + u_{0i}^j + u_{{\epsilon}i}^j \right)\\
			& = \sum_{j=1}^{n_i}\left( - \rho ({\varphi_{i}^j}) + u_{{\epsilon}i}^j \right) ,
	\label{eq:vi}
	\end{split}
\end{align}
%
\begin{figure}[t]
	\centering
		\includegraphics[keepaspectratio, scale=0.33]{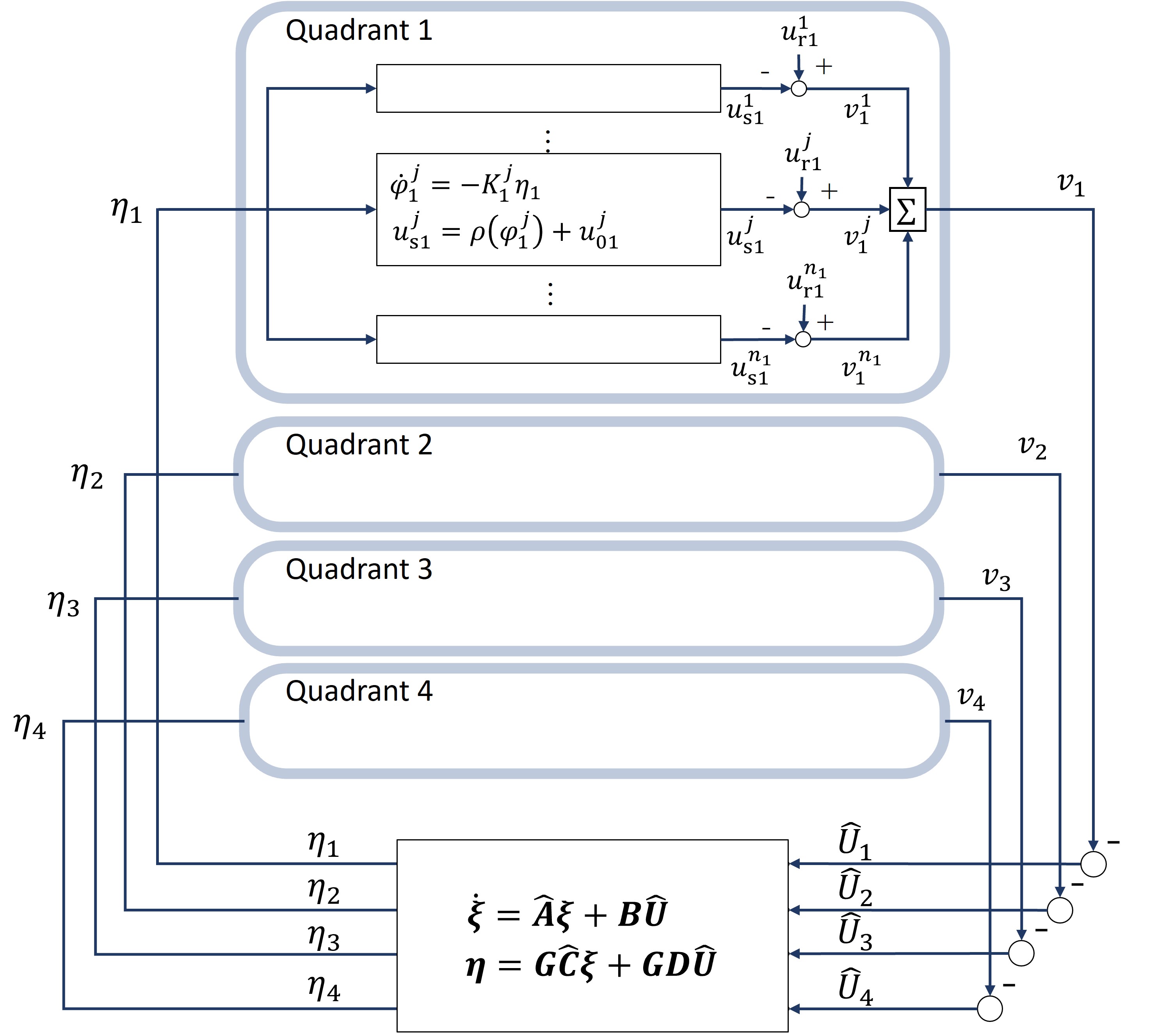}
	  \caption{System}\label{fig:system_p}
\end{figure}
where $u_{{\epsilon}i}^j$ represents the thrust required to move the reference position.
$v_i$ of eq.~\eqref{eq:vi} has the same form as that in \cite{Amano}.
For every agent $j$ in quadrant $i$, there exists a storage function $V_{{\mathrm c}i}^j$ that suffices the following condition\cite[A.20]{Amano}.
\begin{equation}
	\dot{V}_{\mathrm{c}i}^j = v_i^j \eta^i_j + \left(K_i^j(\varphi_i^j,\eta_i^j) L_i^j(\varphi_i^j) -1\right) \left(\rho(\varphi_i^j) - \tilde{u}_{\epsilon i}(\varphi_i^j) \right)\eta_i^j + \left(u_{\epsilon i}^j - \tilde{u}_{\epsilon i}(\varphi_i^j)\right) \eta_i^j ,
	\label{eq:int1_2}
\end{equation}
where $L_i^j(\varphi_i^j)$ and $\tilde{u}_{\epsilon i}^j (\varphi_i^j)$ satisfy the following conditions.
\begin{equation*}
	\begin{split}
		&\underline{k_i}^j \le L_i^j(\varphi_i^j) \le \overline{k_i}^j ,\\
		\tilde{u}_{\epsilon i}^j(\varphi_i^j) = & \begin{cases}
			\min\left(\rho(\varphi_i^j) , u_{\epsilon i}^j\right) & \text{if} \quad u_{\epsilon i}^j \ge 0 \\
			\max\left(\rho(\varphi_i^j) , u_{\epsilon i}^j\right) & \text{if} \quad u_{\epsilon i}^j < 0 
		\end{cases}.
	\end{split}
\end{equation*}
According to \cite{Amano}, the second term in eq.~\eqref{eq:int1_2} is non-positive.
Furthermore, when $\overline{k_i}^j$ is sufficiently large, $\varphi_i^j \eta_i^j \le 0$ holds, rendering the third term approximately non-positve.
Hence, the following relation holds
\begin{equation}
	\dot{V}_{\mathrm{c}i}^j \le v_i^j \eta_i^j .
	\label{eq:int1_3}
\end{equation}
From eq.~\eqref{eq:int1_3}, the storage functions $V_{\mathrm{c}i}$ for each quadrant $i$ satisfy the following conditions:
\begin{equation}
	\dot{V}_{\mathrm{c}i} = \sum_{j=1}^{n_i} \dot{V}_{\mathrm{c}i}^j \le \sum_{j=1}^{n_i} v_i^j \eta_i^j.
	\label{eq:int1_4}
\end{equation}
Therefore, the overall storage function $V_{\mathrm{c}}$ satisfies
\begin{equation}
	\dot{V}_{\mathrm{c}} \le \bm{v}^{\top}\bm{\eta} .
\end{equation}
Thus, according to \cite{van2000}, eq. (\ref{eq:assc}) is passive.

\begin{table*}[p]
	\caption{Rising time. The unit is seconds.}
	\label{table:f}
	\centering
	\begin{tabular}{|c|cccc|cccc|}
		\hline
		Shape & \multicolumn{4}{|c|}{Rectangle-shape} &  \multicolumn{4}{|c|}{L-shaped}\\ \hline
		State & $X$ & $Y$ & $Z$ & $\Psi$ & $X$ & $Y$ & $Z$ & $\Psi$ \\ \hline
		Our controller & $2.85$ & $2.96$ & $2.27$ & $2.18$ & $2.56$ & $2.64$ & $2.18$ & $2.17$ \\
		PID controller & $2.76$ & $2.99$ & $2.13$ & $2.21$ & $2.55$ & $2.73$ & $2.06$ & $2.11$ \\
		\hline
\end{tabular}
\end{table*}

\begin{figure*}[p]
	\centering
	\subfloat[Rectangle-shape $X$]{
		\includegraphics[width=4.6cm]{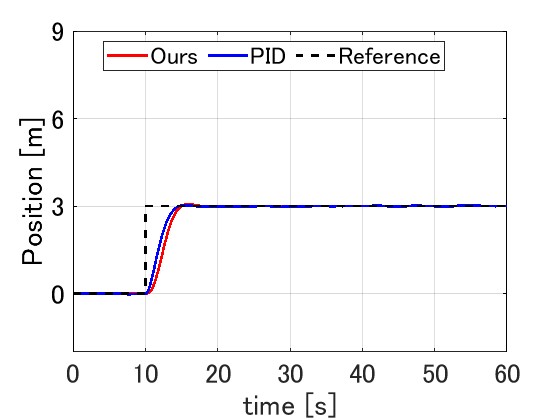}
		\label{fig:step1}
	}
	\subfloat[Rectangle-shape $Y$]{
		\includegraphics[width=4.6cm]{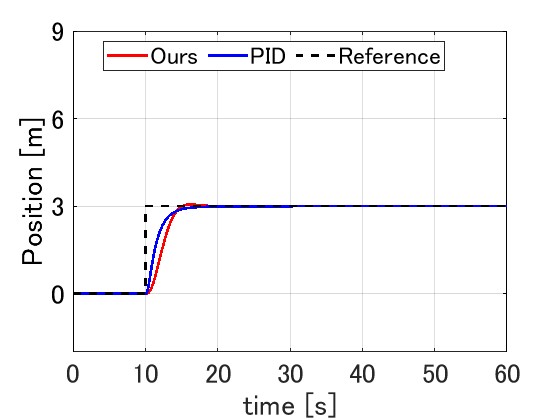}
		\label{fig:step2}
	} \\
	\subfloat[Rectangle-shape $Z$]{
		\includegraphics[width=4.6cm]{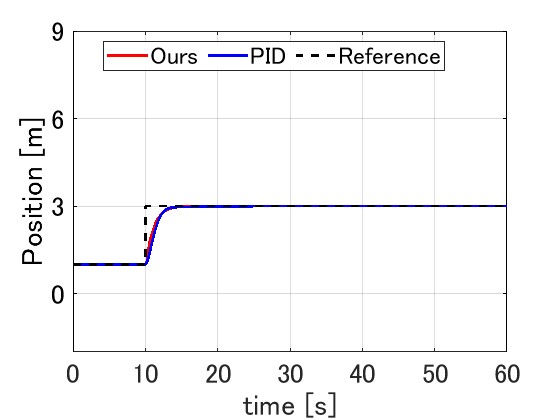}
		\label{fig:step3}
	}
	\subfloat[Rectangle-shape $\psi$]{
		\includegraphics[width=4.6cm]{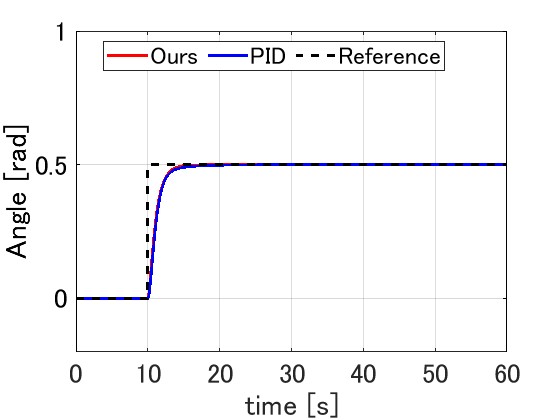}
		\label{fig:step4}
	} \\
	\subfloat[L-shape $X$]{
		\includegraphics[width=4.6cm]{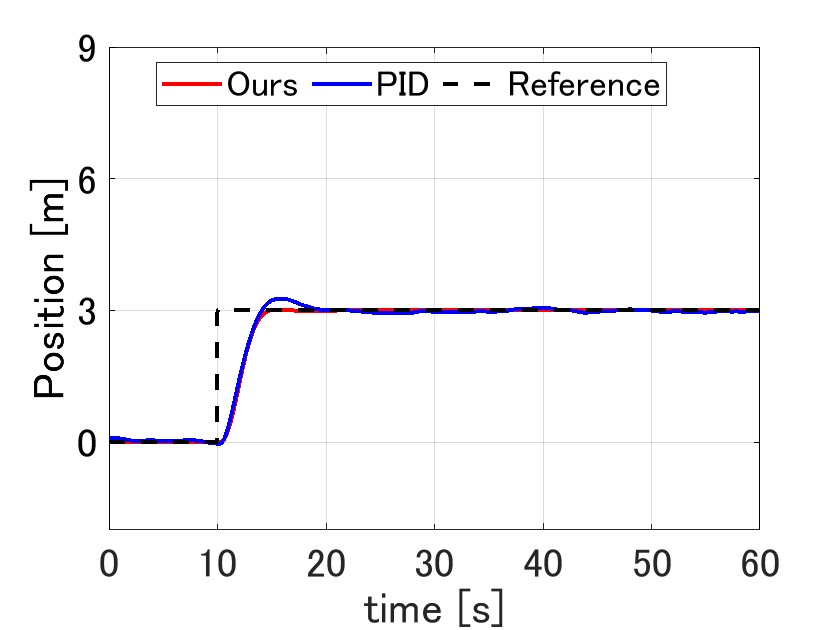}
		\label{fig:step5}
	}
	\subfloat[L-shape $Y$]{
		\includegraphics[width=4.6cm]{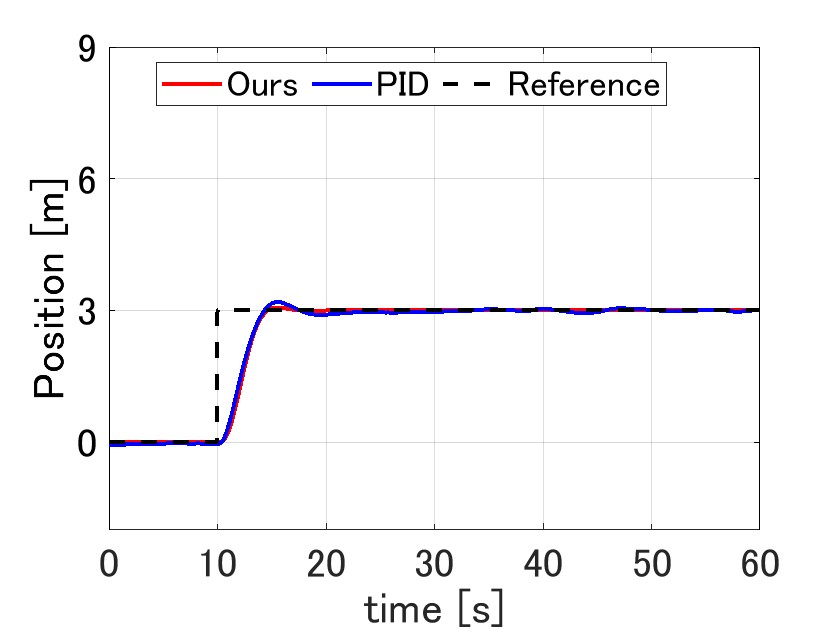}
		\label{fig:step6}
	} \\
	\subfloat[L-shape $Z$]{
		\includegraphics[width=4.6cm]{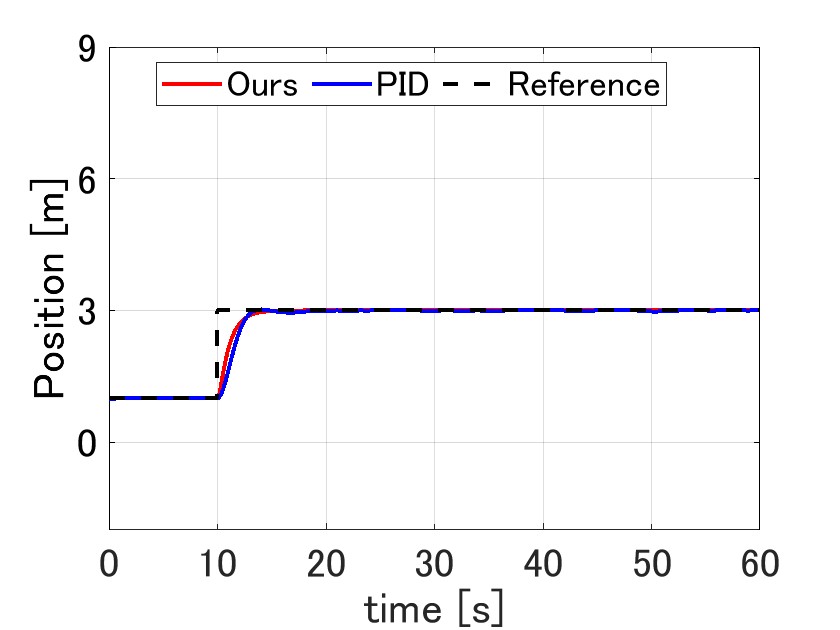}
		\label{fig:step7}
	}
	\subfloat[L-shape $\psi$]{
		\includegraphics[width=4.6cm]{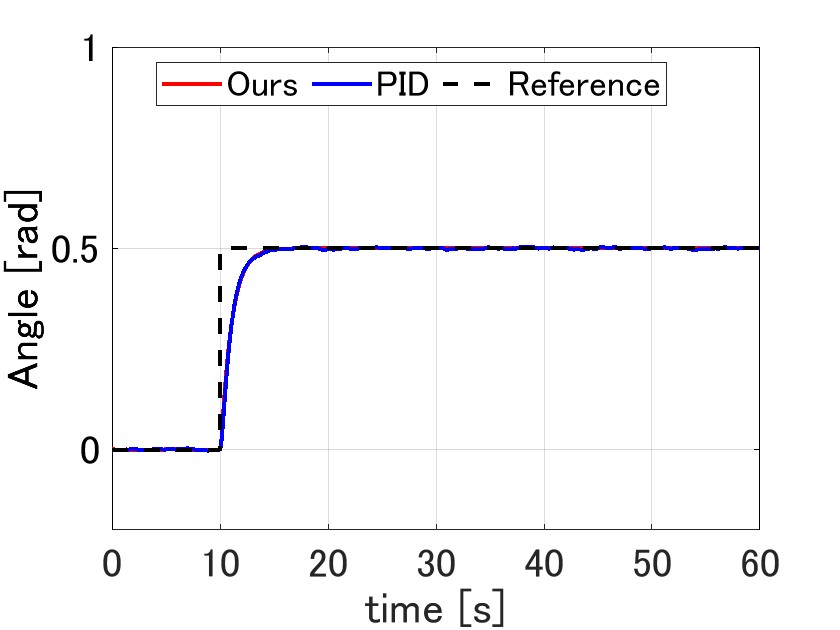}
		\label{fig:step8}
	}

	\caption{Time series of the position, yaw angle, and reference of the
	Rectangular- and L-shaped payloads using the proposed controller and PID controller (PID).}
	\label{fig:step}
\end{figure*}

\section{PID controller}\label{sec_pid}	
The parameters of the PID control, which is the comparison target of the proposed control, were designed to be the same as the rising time of the proposed control.
Herein, the rising time is the time required for the response to increase from 10 \% to 90 \% of the way from the initial to the steady-state value.
Each rising time is presented in Table \ref{table:f}, and a comparison of responses is shown in Fig.~\ref{fig:step}.
We evaluated the step responses from $0.0$ to $3.0$ m for $X$ and $Y$, from $1.0$ to $3.0$ m for $Z$, and from $0.0$ rad to $0.5$ rad for $\psi$.

\end{document}